\documentclass[11pt]{article} 
\usepackage{algorithm, algorithmic,wrapfig}

\usepackage[left=1in,top=1in,right=1in,bottom=1in,letterpaper]{geometry}

\usepackage{microtype}
\usepackage{graphicx}
\usepackage{subfigure}
\usepackage{booktabs} 
\usepackage{color,textgreek}
\usepackage{bm,multirow}
\usepackage{enumerate,mathtools}
\usepackage{ amssymb }
\usepackage{makecell}

\usepackage{amsmath}    
\usepackage[utf8]{inputenc}
\newenvironment{proof}{\par\noindent{\bf Proof\ }}{\hfill\BlackBox\\[2mm]}

\newcommand{\kb}[1]{{\color{red}\bf[KB: #1]}}

\usepackage{amsmath,color}
\usepackage{algorithm}
\usepackage{epstopdf}
\usepackage{amsfonts}
\usepackage{enumitem}

\allowdisplaybreaks
\usepackage{hyperref}
\usepackage{amsmath}    
\usepackage[T1]{fontenc}
\usepackage[utf8]{inputenc}

\usepackage{epsf}

\usepackage{mathtools} 
\usepackage{ mathrsfs }

\DeclareMathOperator*{\argmin}{argmin}

\newcommand{\abs}[1]{\lvert #1 \rvert}
\newcommand{\norm}[1]{\left\lVert #1 \right\rVert}

\newcommand{\calP}{\mathcal{P}}

\newcommand{\expec}[1]{\mathbf{E}\left[ #1 \right] }
\newcommand\numberthis{\addtocounter{equation}{1}\tag{\theequation}}  

\def\fn[#1]#2{{f_{#1}\left(x_{#2}\right)}}

\newtheorem{definition}{Definition}[section]
\newtheorem{lemma}{Lemma}[section]

\newtheorem{theorem}{Theorem}[section]

\newtheorem{assumption}{Assumption}[section]
\newtheorem{remark}{Remark}

\newcommand{\calF}{{\mathcal{F}}}
\newcommand{\order}{\ensuremath{\mathcal{O}}}
\newcommand{\torder}{\ensuremath{\tilde{\mathcal{O}}}}

\newcommand{\cH}{{\mathcal{H}}}

\def\E{{\bf E}}
\def\exp{{\rm exp}}

\def\cF{{\cal F}}

\def\cS{{\cal S}}
\def\cT{{\cal T}}

\def\cH{{\cal H}}

\def\cF{{\cal F}}

\def\cS{{\cal S}}

\def\bg{{\bar{G}_t}}
\def\bh{{\bar{H}_t}}
\def\hx{{\hat x}}

\def\hzt{{\hat \zeta}}
\def\ep{{\epsilon}}

\def\tz{{\tilde{\zeta}}}
\def\lge{{\log\left(\frac{1}{\epsilon}\right)}}

\title{Escaping Saddle-Points Faster under Interpolation-like Conditions}

\author{Abhishek Roy\thanks{Department of Statistics, University of California, Davis. Email: abroy@ucdavis.edu} 
\and Krishnakumar Balasubramanian\thanks{Department of Statistics, University of California, Davis. Email: kbala@ucdavis.edu}
\and Saeed Ghadimi \thanks{Department of Management Sciences, University of Waterloo. Email: sghadimi@uwaterloo.ca}
\and Prasant Mohapatra\thanks{Department of Computer Science, University of California, Davis. Email: pmohapatra@ucdavis.edu.}
}

\begin{document}
\maketitle




\begin{abstract}
In this paper, we show that under over-parametrization several standard stochastic optimization algorithms escape saddle-points and converge to local-minimizers much faster. One of the fundamental aspects of over-parametrized models is that they are capable of interpolating the training data. We show that, under interpolation-like assumptions satisfied by the stochastic gradients in an over-parametrization setting, the first-order oracle complexity of Perturbed Stochastic Gradient Descent (PSGD) algorithm to reach an $\epsilon$-local-minimizer, matches the corresponding deterministic rate of $\torder(1/\epsilon^{2})$. We next analyze Stochastic Cubic-Regularized Newton (SCRN) algorithm under interpolation-like conditions, and show that the oracle complexity to reach an $\epsilon$-local-minimizer under interpolation-like conditions, is $\torder(1/\epsilon^{2.5})$. While this obtained complexity is better than the corresponding complexity of either PSGD, or SCRN without interpolation-like assumptions, it does not match the rate of $\torder(1/\epsilon^{1.5})$ corresponding to deterministic Cubic-Regularized Newton method. It seems further Hessian-based interpolation-like assumptions are necessary to bridge this gap. We also discuss the corresponding improved complexities in the zeroth-order settings.
\end{abstract}

\section{Introduction}
Over-parametrized models, for which the training stage involves solving nonconvex optimization problems, are common in modern machine learning. A canonical example of such a model is deep neural networks. Such over-parametrized models have several interesting statistical and computational properties. On the statistical side, such over-parametrized models are highly expressive and are capable of nearly perfectly interpolating the training data. Furthermore, despite the highly nonconvex training landscape, most local minimizers have good generalization properties under regularity conditions; see for example~\cite{nguyen2017loss,kawaguchi2016deep,haeffele2015global,haeffele2014structured} for empirical and theoretical details. We emphasize here that over-parametrization plays an important role for both phenomenon to occur. Furthermore, it is to be noted that not all critical points exhibit nice generalization properties. Hence, from a computational perspective, designing algorithms that do not get trapped in saddle-points, and converge to local minimizers during the training process, becomes extremely important~\cite{dauphin2014identifying}. 

Indeed, recently there has been extensive research in the machine learning and optimization communities on designing algorithms that escape saddle-points and converge to local minimizers. The authors of~\cite{lee2016gradient} proved the folklore result that in the deterministic setting for sufficiently regular functions, vanilla gradient descent algorithms converges almost surely to local minimizers, even when initialized randomly; see also~\cite{lee2017first}. However,~\cite{lee2016gradient,lee2017first} only provide asymptotic results, that have limited consequence for practice. Understandably, it has been shown by the authors of ~\cite{du2017gradient}, that gradient descent might take exponential-time to escape saddle points in several cases. In this context, injecting artificial noise in each step of the gradient descent algorithm has been empirically observed to help escape saddle points. Several works, for example,~\cite{jin2017escape, jin2018accelerated}, showed that such \emph{perturbed} gradient descent algorithms escape saddles faster in a non-asymptotic sense. Such algorithms are routinely used in training highly over-parametrized deep neural network and other over-parameterized nonconvex machine learning models.  However, existing theoretical analysis of such algorithms fail to take advantage of the interpolation-like properties enjoyed by over-parametrized machine learning models. Hence, such theoretical results are conservative. Specifically, there is a gap between the assumptions used in the theoretical analysis of algorithms that escape saddle-points and the assumptions commonly satisfied by over-parametrized models which are trained by those algorithms.

In this work, we consider nonconvex stochastic optimization problems of the following form:
\begin{align}\label{eq:mainproblemstopt}
\argmin_{x \in \mathbb{R}^d} \left\{f(x): = \E_\xi [F(x,\xi)] \right\}.
\end{align}
where $f:\mathbb{R}^d\to \mathbb{R}$ is nonconvex function satisfying certain regularity properties described next, and $\xi$ is a random variable characterizing the stochasticity in the problem. We assume that the function $f$ has a lower bound $f^*$ throughout this work. We  analyze two standard algorithms that escape saddle-points, namely the perturbed stochastic gradient descent (PSGD) and stochastic cubic-regularized Newton's method (SCRN) for problems of the form in~\eqref{eq:mainproblemstopt}. We show that under interpolation-like assumptions (see Section~\ref{sec:prelim} for exact definitions) on the stochastic gradient, it could be proved that both PSGD and SCRN escape saddle-points and converge to local minimizers much faster. In particular, we show that in order for PSGD algorithm to escape saddle-points and find an $\epsilon$-local-minimizer, the number of calls to the stochastic first-order oracle is of the order $\torder(1/\epsilon^{2})$\footnote{Here, $\torder$ hides $\log$ factors.} which matches number of calls when the objective being optimized is a deterministic objective (for which exact gradient could be obtained in each step of the algorithm)\footnote{It is possible to obtain $\torder(1/\epsilon^{11.75})$ complexity using accelerated method in deterministic setting; see~\cite{jin2018accelerated}. }. As a point of comparison,~\cite{ge2015escaping, jin2019nonconvex} showed that without the interpolation-like conditions that we make, PSGD requires~$\torder(1/\epsilon^{4})$ calls to the stochastic gradient oracle. Furthermore,~\cite{fang2019sharp} analyzed a version of PSGD with averaging and improved the oracle complexity to $\torder(1/\epsilon^{3.5})$. It is also worth noting that, with a mean-square Lipschitz gradient assumption on the objective function being optimized, and using complicated variance reduction techniques, the authors of~\cite{fang2018spider} showed that it is possible for a double-loop version of PSGD to converge to $\epsilon$-local minimizers with $\torder(1/\epsilon^{3})$ number of calls to the stochastic first-order oracle. However, recent empirical investigations seem to suggest that variance reduction techniques are inefficient for the nonconvex deep learning problems~\cite{defazio2019ineffectiveness,schmidt2020}. Our results, on the other hand exploit the naturally available structure present in over-parametrized models and obtains the best-known oracle complexity for escaping saddle-points using only the vanilla versions of PSGD algorithm (which is oftentimes the version of PSGD used in practice). We also analyze the corresponding Zeroth-Order  version of the PSGD algorithm. In this setting, we are able to observe only potentially noisy evaluations of the function being optimized. In this setting, we show that PSGD algorithm requires $\torder(d^{1.5}/\epsilon^{4.5})$ calls to the stochastic zeroth-order oracle. In this context, we are not aware of a result to compare with. The recent works of~\cite{bai2019escaping, flokas2019efficiently}  provided results for bounded functions in the zeroth-order deterministic setting, where one obtains exact function values; such a setting though is highly unrealistic in practice.

Next, we consider the question of whether using second-order methods helps reduce the number of calls. Indeed, in the deterministic setting, it is well-known that second-order information helps escape saddle point at a much faster rate. For example,~\cite{nesterov2006cubic} proposed that Cubic-regularized Newton's method and showed that the method requires only $\torder(1/\epsilon^{1.5})$ calls to the gradient and Hessian oracle; see also~\cite{cartis2011adaptive, curtis2017trust} for related results. Correspondingly, in the stochastic setting~\cite{tripuraneni2018stochastic} showed that SCRN method requires $\torder(1/\epsilon^{3.5})$ calls, which is better than that of PSGD (without further assumptions). In this work, we show that under interpolation-like assumptions on (only) the stochastic gradient, SCRN method requires only $\torder(1/\epsilon^{2.5})$ calls. In contrast to the PSGD setting, SCRN requires more calls than its corresponding deterministic counterpart. However, it should be noted that the complexity of SCRN is still better than that of the PSGD, with or without interpolation-like assumptions. We belive that without further interpolation-like assumptions also on the stochastic Hessians, the oracle complexity of  SCRN cannot be improved, in particular to match the deterministic rate of $\torder(1/\epsilon^{1.5})$ (see also Remark~\ref{scrnrateimprovement}). We also provide similar improved results for a zeroth-order version of SCRN method, thereby improving upon the results of~\cite{balasubramanian2018zeroth}. All of our results, along with comparison to existing results in the literature and the corresponding assumption required, are summarized in Table~\ref{tab:complexity}. We conclude this section with a other related works.

\vspace{0.1in}

\textbf{More Related Works.} In the interpolation regime,~\cite{ma2018power} recently showed that mini-batch stochastic gradient descent (SGD) algorithm enjoys exponential rates of convergence for unconstrained strongly-convex optimization problems; see also~\cite{strohmer2009randomized, needell2014stochastic} for related earlier work. For the non-convex setting, \cite{bassily2018exponential} analyze SGD for non-convex functions satisfying the Polyak-Lojasiewicz (PL) inequality \cite{polyak1963gradient} under the interpolation condition and show that SGD can achieve a linear convergence rate. Recently, \cite{vaswani2018fast} introduced a more practical form of interpolation condition, and prove that the constant step-size SGD can obtain the optimal convergence rate for strongly-convex and smooth convex functions. They also show the first results in the non-convex setting that the constant step-size SGD can obtain the deterministic rate in the interpolation regime for converging to first-order stationary solution. Subsequently, \cite{meng2019fast} investigate the regularized subsampled Newton method (R-SSN) and the stochastic BFGS algorithm under the interpolation-like conditions. We emphasize that all the above works consider the case of convex objective function predominantly; the only exception is~\cite{vaswani2018fast} that consider the nonconvex case but only study convergence to first-order stationary solution. There has been several works on obtaining oracle complexity of escaping saddle-points in the finite-sum setting; we refer the interested reader to~\cite{allen2017katyusha,zhou2019stochastic,zhang2018adaptive,wang2018stochastic} and references therein for such results. We emphasize that a majority of the above works are based on complicated variance reduction techniques that increase the implementation complexity of such methods and make them less appealing in practice. There exist only few works on  escaping saddle-points for constrained optimization problems; see~\cite{lu2019perturbed,lu2019snap,nouiehed2019trust,mokhtari2018escaping} for more details.

We also briefly discuss the consequences of our results to deep neural network training and related works. Roughly speaking, there are now two potential explanations for the success of optimization methods for training deep neural networks~\cite{sun2019optimization}. The first explanation is based on landscape analysis. This involves two steps: Showing the optimization landscape has favorable geometry~\cite{kawaguchi2020elimination} (i.e., all local minima are (approximate) global minima under suitable regularity conditions), and hence constructing optimization algorithms that can efficiently escape saddle-points. The second explanation is based on the NTK viewpoint; see, for example \cite{jacot2018neural, chizat2018global, chizat2019lazy, allen2019convergence, du2019gradient,  zou2020gradient}, for a partial overview. However, a majority of the results based on NTK viewpoint are for polynomially (in depth and sample-size) large-width networks (indeed,~\cite{allen2019convergence} mention that their polynomial degrees are impractical).  Our results in this paper are geared towards the former program. 

\begin{table*}[t]\label{tab:complexity}
	\begin{center}
	\resizebox{\textwidth}{!}{\begin{tabular}{|c|c|c|cc|c|}
		\hline
		\multirow{2}{*}{Algorithm} & \multicolumn{2}{c|}{\makecell{\bf{With SGC}\\(\bf{This paper})}} & \multicolumn{2}{c|}{Without SGC} & {Deterministic} \\ \cline{2-6} 
		& ZO            & HO            & \multicolumn{1}{c|}{ZO}   & HO  &  HO              \\ \hline
		\makecell{Perturbed\\ GD}               &   \makecell{$\bm{\torder\left({d^{1.5}}{\ep^{-4.5}}\right)}$\\Theorem~\ref{th:pgdmaintheorem}}            &        \makecell{$\bm{\torder\left(\ep^{-2}\right)}$\\Theorem~\ref{th:pgdmaintheorem}}       &  \multicolumn{1}{c|}{ \makecell{$\bm{\torder\left({d^{1.5}}{\ep^{-5.5}}\right)}$\\Theorem~\ref{th:pgdmaintheoremwosgc}}          }     & \makecell{$\torder\left(\ep^{-4}\right)$\\Theorem 17 \\\cite{jin2019nonconvex}}    &            \makecell{$\torder\left(\ep^{-2}\right)$\\Theorem 3\\\cite{jin2017escape}}           \\ \hline
		\makecell{Cubic \\ Newton}               &      \makecell{$\bm{\torder\left({d^4}{\ep^{-2.5}}\right)}$\\Theorem~\ref{th:crnmaintheorem} }        &  \makecell{$\bm{\torder\left(\ep^{-2.5}\right)}$\\Theorem~\ref{th:crnmaintheorem}}              & \multicolumn{1}{l|}{\makecell{$\torder\left({d^{4}}{\ep^{-2.5}}\right)+\order\left({d}{\ep^{-3.5}}\right)$\\Theorem 4.1\\ \cite{balasubramanian2018zeroth}}}     &  \makecell{$\torder\left(\ep^{-3.5}\right)$ \\Theorem 1\\ \cite{tripuraneni2018stochastic}}                 &  \makecell{$\order\left(\ep^{-1.5}\right)$\\Theorem 3\\\cite{nesterov2006cubic}}                \\ \hline
	\end{tabular}}
\end{center}
	\caption{Oracle complexities of perturbed stochastic gradient descent (PSGD) and stochastic cubic-regularized Newton's method (SCRN). ZO corresponds to number of calls to zeroth-order oracle and HO corresponds to number of calls to first or second-order oracles. The result for PSGD and SCRN are given respectively in high-probability and in expectation.  The results in the deterministic case corresponds to projected gradient descent and cubic-Regularized Newton's method (without stochastic gradients).}
\end{table*}

\section{Preliminaries}\label{sec:prelim}
We now present the assumptions and definitions used throughout the paper. Section-specific additional details are in the respective sections. In this paper we use $\|\cdot\|$, and $\|\cdot\|_*$ to denote a norm and the corresponding dual norm on $\mathbb{R}^d$.
We now describe some regularity conditions made on the objective function in~\eqref{eq:mainproblemstopt} assumptions in this work.
\begin{assumption}[Lipschitz Function] \label{as:lip}
The function $F$ is $L$-Lipschitz, almost surely for any $\xi$, i.e., $|F(x,\xi)- F(y,\xi)| \leq L\norm{x-y}$. Here we assume $\|\cdot\|=\|\cdot\|_2$, unless specified explicitly.
\end{assumption}
\begin{assumption}[\bfseries Lipschitz Gradient]  \label{as:lipgrad}
The function $F$ has Lipschitz continuous gradient, almost surely for any $\xi$, i.e., $\norm{\nabla F( x,\xi) -\nabla F( y,\xi)}\leq L_{G}\|x-y\|_*$, where $\|\cdot\|_*$ denotes the dual norm of $\|\cdot\|$. This also implies $\abs{F( y,\xi)-F( x,\xi)-\nabla F(x,\xi)^\top (y-x)}\leq \frac{L_G}{2}\|y-x\|^2$.
\end{assumption}
\begin{assumption}[\bfseries Lipschitz Hessian]  \label{as:liphess}
The function $F$ has Lipschitz continuous Hessian, almost surely for any $\xi$, i.e.,$\norm{\nabla^2 F( x,\xi) -\nabla^2 F( y,\xi)}\leq L_H\|x-y\|$.
\end{assumption}
Note that if Assumptions~\ref{as:lip}--\ref{as:liphess} are true for $F$, then they also hold for $f(\cdot)=\expec{F(\cdot,\xi)}$; but the other way around is not true. For our higher-order results, we make the above assumptions only on $f(\cdot)$, which is a weaker assumption. In the interpolation regime, the stochastic gradients become small when the true gradient is small. The following condition, known as Strong Growth Condition (SGC) \cite{vaswani2018fast}, captures how fast the stochastic gradient goes to $0$ with respect to the true gradient. 
\begin{assumption}[SGC~\cite{vaswani2018fast}]\label{as:SGC}
For any point $x \in \mathbb{R}^d$, we have the stochastic gradient satisfy $\mathbf{E}_\xi\|\nabla F(x,\xi)\|^2\leq \rho \|\nabla f(x)\|^2$,~for $\rho >1$. Note here that $\rho =1$, corresponds to the deterministic setting.
\end{assumption}
SGC controls the variance of the obtained stochastic gradient in the above mentioned way. Note in particular that in the case when $\|\nabla f(x)\|^2=0$, under SGC, we have almost surely $\|\nabla F(x,\xi)\|^2=0$. This means that when the point $x$ is a stationary point of the function $f$, then it is also a stationary point of the function $F$ almost surely. In the context of deep neural networks, the function $F$ corresponds to the risk based on training sample $\xi$ and the function $f$ corresponds to the risk. Hence, the strong growth condition states that that deep neural network is capable of interpolating the training data almost surely. Such a phenomenon is observed in practice with deep neural networks, which provides a strong motivation for using this assumption for analyzing the performance of PSGD and SCRN for escaping saddle-points.

In this work, we study the algorithms under two oracles settings:  Stochastic zeroth-order oracle, where one obtains noisy unbiased function evaluations, and the stochastic higher-order oracle, where one obtains noisy unbiased estimators of the gradients, and hessians. We now define them formally. 
\begin{assumption}[Zeroth-order oracle]\label{as:stochzero}
	For any $x\in \mathbb{R}^d$, the zeroth order oracle outputs an estimator $F(x,\xi)$ of $f(x)$ such that $\expec{F(x,\xi)}=f(x)$, $\expec{\nabla F(x,\xi)}=\nabla f(x)$, $\expec{\nabla^2 F(x,\xi)}=\nabla^2 f(x)$, and $\expec{\|\nabla^2 F(x,\xi)-\nabla^2 f(x)\|_F^4}\leq \sigma_2^4$, where $\|\cdot\|_F$ is the Frobenius norm.  
\end{assumption}
\begin{assumption}[Higher-order oracles]\label{as:stochfirst}
	For any $x\in \mathbb{R}^d$, (i) the first-order oracle outputs an estimate $\nabla F(x,\xi)$ of $\nabla f(x)$ such that $\expec{\nabla F(x,\xi)}=\nabla f(x)$ and (ii) the second-order oracle, in addition outputs an estimate $\nabla^2 F(x,\xi)$  of $\nabla^2 f(x)$ such that, $\expec{\nabla^2 F(x,\xi)}=\nabla^2 f(x)$, and $\expec{\|\nabla^2 F(x,\xi)-\nabla^2 f(x)\|_F^4}\leq \sigma_2^4$.  
\end{assumption}
Such assumptions on the zeroth-order and higher-order oracles are standard in the literature; see for example~\cite{ghadimi2013stochastic, nesterov2017random, balasubramanian2018zeroth}. Our goal in this paper is to reach an approximate local minimizer (also called as a second-order stationary point) of a non-convex function, which is defined as follows:
\begin{definition}[$\ep$-Local Minimizer]
Let Assumption~\ref{as:liphess} hold for a function $f$. Then a point $\bar{x}$ is called a $\epsilon$-second-order stationary point if,
\begin{align}
   \max\left(\sqrt{\|\nabla f\left(\bar{x}\right)\|},-\frac{\lambda_{min}\left(\nabla^2 f\left(\bar{x}\right)\right)}{L_H}\right)\leq \sqrt{\epsilon} \label{eq:secondorderdef}
\end{align}
where $\lambda_{min}\left(\nabla^2 f\left(\bar{x}\right)\right)$ is the minimum eigenvalue of $\nabla^2 f\left(\bar{x}\right)$. 
\end{definition}
Note that for stochastic optimization problems, the quantity on the left hand side of~\eqref{eq:secondorderdef}, is a random variable. In this paper we prove a high-probability bound, and an expectation bound for the above quantity for PSGD, and SCRN respectively. 

For a point $x_t$, we will use $\nabla_t$, $\nabla_t^2$, $h_t$, and $\lambda_{1,t}$ to denote $\nabla_t$, $\nabla^2 f(x_t)$, $(x_{t+1}-x_t)$,and $\lambda_{min}\left(\nabla^2 f\left(x_t\right)\right)$ respectively. The zeroth-order minibatch gradient \cite{nesterov2017random}, and Hessian estimator \cite{balasubramanian2018zeroth} $g_t$, and $H_t$ are defined as:
\begin{align}\label{eq:zerohessdef} 
g_t=\frac{1}{n_1}\sum_{i=1}^{n_1}\frac{F(x_t+\nu u_{t,i},\xi_{t,i})-F(x_t,\xi_{t,i})}{\nu}u_i,\qquad  \quad H_t=\frac{1}{n_2}\sum_{i=1}^{n_2}\mathfrak{H}_{t,i}\left(u_{t,i}u_{t,i}^\top-I\right),
\end{align}
where $$\mathfrak{H}_{t,i}=\frac{F(x_t+\nu u_{t,i},\xi_{t,i})+F(x_t-\nu u_{t,i},\xi_{t,i}) -2F(x_t,\xi_{t,i})}{2\nu^2},$$ and $u_{t,i}\sim \mathcal{N}\left(\pmb {0},\pmb{I}_d\right)$ $\forall t=1,2,\cdots,T, i=1,2,\cdots,n_1$. We will use $\zeta_t=g_t-\nabla_t=\frac{1}{n_1}\sum_{i=1}^{n_1}g_{t,i}-\nabla_t$, and $\tz_t=\zeta_t+\theta_t$. In the following lemma we show that under SGC, the variance of $\nabla F(x_t,\xi)$ is of the order of the gradient norm squared. 
\begin{lemma}\label{lm:gradestvar}
Let Assumption~\ref{as:SGC} hold for a function $f$. Then, for both zeroth-order, and first-order oracle, we have, 
\begin{align}\label{eq:gradestvar}
    \expec{\left\lVert\frac{1}{n_1}\sum_{i=1}^{n_1}\nabla F(x_t,\xi_i)-\nabla_t\right\rVert^2}\leq \frac{\rho - 1}{n_1}\|\nabla_t\|^2.
\end{align}
\end{lemma}
\begin{proof}Let $g_t=\frac{1}{n_1}\sum_{i=1}^{n_1}\nabla F(x_t,\xi_i)$. Then we have
\begin{align*}
\expec{\|g_t-\nabla_t\|^2}=&  \expec{\|g_t\|^2+\|\nabla_t\|^2-2g_t^\top \nabla_t}= \frac{1}{n_1^2}\expec{\norm{\sum_{i=1}^{n_1}\nabla F(x_t,\xi_i)}^2} -\|\nabla_t\|^2\\
	\leq & \frac{1}{n_1^2}\left(\rho n_1\norm{\nabla_t }^2+n_1(n_1-1)\norm{\nabla_t}^2\right)-\|\nabla_t\|^2= \frac{\rho - 1}{n_1}\|\nabla_t\|^2,
	\end{align*}
	which completes the proof.
\end{proof}
\begin{remark}
The above simple results actually turns out to have far-reaching consequences for obtaining improved complexity bounds for both PSGD and SCRN algorithms. It implies that when the true gradient is small, the variance of the stochastic gradient is also small. Typically, in the analysis of PSGD and SCRN, it is assumed that the stochastic gradients are assumed to have a constant variance. But for over-parametrized models, we will use Lemma~\ref{lm:gradestvar} to prove deterministic rate for PSGD and  improved rates for SCRN. 
\end{remark}
\section{Perturbed Stochastic Gradient Descent}
\begin{algorithm}[t]
	\caption{Perturbed Stochastic Gradient Descent Algorithm }\label{alg:pgd}	
	{\bf{Input:}} $x_0\in \mathbb{R}^d$, $\eta$, $r$.\\
	{\bf for} $t=0$ to $T$ \bf do	\\
	{\bf Set} $g_t=\frac{1}{n_1}\sum_{i=1}^{n_1}g_{t,i}$ where 
	\begin{align*}
	&g_{t,i}=\nabla F\left(x_t,\xi_{t,i}\right)  \tag*{(First-order)}\\
	&g_{t,i}=\frac{F(x_t+\nu u_{t,i},\xi_{t,i})-F(x_t,\xi_{t,i})}{\nu}u_i \tag*{(Zeroth-order)}
	\end{align*}
	and $u_{t,i}\sim \mathcal{N}\left(\pmb {0},\pmb{I}_d\right)$ $\forall t=1,2,\cdots,T, i=1,2,\cdots,n_1$\\
	{\bf Sample} $\theta_t \in \mathcal{N}\left(\pmb {0},r^2\pmb{I}_d\right)$\\
	{\bf Update} $x_{t+1}=x_t-\eta\left(g_t+\theta_t\right)$\\
	{\bf end for}
\end{algorithm}
In this section we show that under SGC, PSGD attain deterministic rate in the first-order setting and obtains much better rate than previously known rates in the zeroth-order setting. An intuitive explanation of this phenomenon is as follows: in the general stochastic setting,  at time $t$ where $\|\nabla_t\|\geq \ep$, PSGD does not descend as much as in the deterministic setting due to noisy gradient. So it takes more iterations to average out the noise. While escaping a saddle point, due to noisy gradient, the iterates follow the direction of the most negative curvature with more difficulty leading to higher complexity. Under SGC, when $\|\nabla_t\|\geq \ep$, the noise variance is of the order of $\|\nabla_t\|^2$ as shown in Lemma~\ref{lm:gradestvar}. So the algorithm still manages to descent. While escaping a saddle point under SGC, as $\|\nabla_t\|\leq \ep$, and  the gradient noise is also small leading to deterministic rates. 

The outline of the proof of the bounds for PSGD in the first-order setting is similar to \cite{jin2019nonconvex} except that we analyze PSGD under interpolation regime. At a high level the proof has two stages: firstly, we show that when $\|\nabla_t\|\geq \ep$, the function descends as fast as the deterministic case; Secondly, when $\|\nabla f(x_t)\|\leq \ep$, and $\lambda_{min}(\nabla^2 f(x_t))\leq -\sqrt{L_H\ep}$, i.e., $x_t$ is a saddle point, by a coupling argument it is shown that either the function descends or the sequence of iterates are stuck around the saddle point. But then it is shown that the stuck region is narrow enough so that the iterates escape the saddle points with high probability. We now require a condition on the tail of the stochastic gradient.
\begin{assumption}\label{as:subgauss}
For any $x\in\mathbb{R}^d$, $\mathbb{P}\left(\|\nabla F(x,\xi)-\nabla f(x)\|\geq \tau\right)\leq 2e^{-\frac{\tau^2}{2\expec{\|\nabla F(x,\xi)-\nabla f(x)\|^2}}}$.
\end{assumption}
Such light-tail conditions are common in the stochastic optimization literature to obtain high-probability bounds; see for example~\cite{ghadimi2013stochastic, jin2019nonconvex}. Note that under Assumption~\ref{as:SGC}, Assumption~\ref{as:subgauss} is equivalent to
\begin{align*}
\mathbb{P}\left(\|\nabla F(x,\xi)-\nabla f(x)\|\geq \tau\right)\leq 2e^{-\tau^2/(2(\rho-1)\|\nabla f(x)\|^2)} \numberthis\label{eq:sgsgc}
\end{align*}
We now present our main result on PSGD.
\begin{theorem}\label{th:pgdmaintheorem}
\begin{enumerate}[label=\alph*), leftmargin=10pt,noitemsep]
	\item Under Assumptions~\ref{as:lipgrad}, \ref{as:liphess} on the function $f(\cdot)$, and Assumptions \ref{as:SGC}, and \ref{as:subgauss}, choosing,
	\begin{align*}
\hspace{-0.5in}\eta={\lge^{-2}}\bigg/{a_0\log\left(\frac{f(x_0)-f^*}{\delta\ep}\right)}, ~ r=\ep^{1.5} \log(\epsilon^{-1})^{-3}, n_1=512c(\rho-1)\log(\ep^{-1}),\numberthis\label{eq:allparameterchoicefirst}
	\end{align*}
	with probability at least $1-\delta$, half of the iterations of Algorithm~\ref{alg:pgd} will be $\ep$-local minimizers after $T$ iterations where,
	\begin{align*}
T=a_1\max\left\lbrace\frac{(f(x_0)-f^*){\cT_{1}}}{\cF_1},\frac{(f(x_0)-f^*)}{\eta \ep^2}\right\rbrace=\tilde{\mathcal{O}}\left(\frac{\log\left(\frac{1}{\delta}\right)}{\ep^2}\right), \numberthis\label{eq:Tchoicefirst}
	\end{align*}
	where $a_0,~a_1$ are constants, and ${\cT_{1}}={0.5\lge^3}/{\sqrt{\ep}}$, and $ \cF_1={\ep^{1.5}}/{\lge^7}$. 
	\item Under Assumptions~\ref{as:lip},~\ref{as:lipgrad}, \ref{as:liphess}, \ref{as:SGC}, and \ref{as:subgauss} in the zeroth order-setting, choosing,
	\begin{align*}
	\eta=\frac{\kappa_0}{\log\left(\frac{f(x_0)-f^*}{\delta \ep}\right)} \quad r=\kappa_1\ep \quad  \nu=\frac{\kappa_4 \ep}{d\lge} 
	\quad n_1=\frac{\kappa_5\lge^{5}d^{1.5}\sqrt{\rho-1}}{\ep^{2.5}}  \numberthis\label{eq:allparameterchoicezero}
	\end{align*}
	with probability at least $1-\delta$, half of the iterations of Algorithm~\ref{alg:pgd} will be $\ep$-local minimzers, after $T$ iterations, where,
	\begin{align*}
	T=\kappa_9\max\left\lbrace\frac{(f(x_0)-f^*){\cT_{0}}}{\cF_0},\frac{(f(x_0)-f^*)}{\eta \ep^2}\right\rbrace
	=\tilde{\mathcal{O}}\left(\frac{\log\left(\frac{1}{\delta}\right)}{\ep^{2}}\right).\numberthis\label{eq:Tchoicezero}
	\end{align*}
	Here, $\kappa_i, i=1,2,\cdots,9$ are absolute constants, and
${\cT_{0}}=\kappa_3\frac{\lge^2\log\left(d\right)^2}{\sqrt{\ep}}$, and $\cF_0=\kappa_8\ep^{1.5}$. Hence, the total number of zeroth-order oracle calls is $Tn_1=\tilde{\mathcal{O}}\left(\frac{d^{1.5}\sqrt{\rho-1}}{\ep^{4.5}}\right)$.
\end{enumerate}
\end{theorem}
\begin{remark}
Note that the complexity result in~\eqref{eq:Tchoicefirst} for the PSGD in the first-order setting matches corresponding complexity of perturbed gradient descent on deterministic optimization problems.
\end{remark}
\begin{remark}
We briefly highlight on the difficulty associated with proving the result in~\eqref{eq:Tchoicezero}. 
First note that in the first-order proof, and also in \cite{jin2019nonconvex}, it is assumed that the noise $\xi$ is sub-gaussian. But for the zeroth-order gradient $g_t$ as defined in \eqref{eq:zerohessdef}, $\|g_t-\nabla_t\|$ no longer has sub-Gaussian tails. Also note that we have from \cite{nesterov2017random}, $\mathbf{E}_{u_{t,i}}[g_{t,i}]=\nabla F_{\nu}(x_t,\xi_{t,i})=\nabla \mathbf{E}_{u_{t,i}}\left[F(x+\nu u_{t,i},\xi_{t,i})\right]$. So $g_t$ is not an unbiased estimator of $\nabla_t$. But as shown in \cite{nesterov2017random}, $\nabla F_{\nu}(x_t,\xi_{t,i})$ is close to $\nabla F(x_t,\xi_{t,i})$. So we first need to establish concentration properties for $g_t$ in the zeroth-order setting. Towards this,  we show that $g_t$ is $\alpha$-sub-exponential with $\alpha=2/3$, even if $\xi$ is sub-gaussian, i.e., the noise in the gradient estimates has heavier tail (Lemma~\ref{lm:zerogradconc}). This leads to the obtained complexity bounds in~\eqref{eq:Tchoicezero}. 
\end{remark}
\begin{remark}
Note that $\cT_{1}$ and $\cT_{0}$ are the number of iterations required to descend by $\cF_{1}$ and $\cF_{0}$ respectively in the first and zeroth-order setting, after the algorithm hits a saddle point. As shown in \cite{jin2019nonconvex}, without SGC, $\cT_{1}=\tilde{O}(\ep^{-2.5})$. In this paper we show that, under SGC, $\cT_{1}=\cT_{0}=\tilde{O}(\ep^{-0.5})$. This shows under SGC, it is indeed possible to escape saddle point faster.
\end{remark}

We highlight here that~\cite{bai2019escaping, flokas2019efficiently} recently considered escaping saddle points in the zeroth-order setting. However they assume that the function being optimized is deterministic (which means exact gradients could be obtained) and is bounded (which means sub-Gaussian tails are possible for the zeroth-order gradient estimator). These two assumptions are however highly impractical and are not satisfied by several situations in practice where zeroth-order optimization techniques are utilized. To the best of our knowledge, there is no known bound on the number of times zeroth-order oracle should accessed for \eqref{eq:Tchoicezero} to hold, when SGC does not hold and only the following standard variance assumption on the unseen stochastic gradient holds  (see, e.g.,~\cite{ghadimi2013stochastic})  for some $\sigma>0$,
\begin{align}\label{eq:gradestvarwosgc}
\expec{\left\lVert\frac{1}{n_1}\sum_{i=1}^{n_1}\nabla F(x_t,\xi_i)-\nabla_t\right\rVert^2}\leq \frac{\sigma^2}{n_1}.
\end{align}
For completeness we present the corresponding result below, which serves as a reference to compare our results with the SGC assumption to  what one could obtain without it.

\begin{theorem}\label{th:pgdmaintheoremwosgc}
	Under Assumptions~\ref{as:lip}, \ref{as:lipgrad}, \ref{as:liphess}, \ref{as:SGC}, and \ref{as:subgauss}, we have the following:
	In the zeroth order-setting, choosing,
		\begin{align*}
		\eta=\frac{\kappa_0}{\log\left(\frac{f(x_0)-f^*}{\delta \ep}\right)} \quad r=\kappa_1\ep \quad  \nu=\frac{\kappa_4 \ep}{d\lge} 
		\quad n_1=\frac{\kappa_5\lge^{5}d^{1.5}\sigma}{\ep^{3.5}}  \numberthis\label{eq:allparameterchoicezerowosgc}
		\end{align*}
		with probability at least $1-\delta$, half of the iterations of Algorithm~\ref{alg:pgd} will be $\ep$-local minimzers, after $T$ iterations, where,
		\begin{align*}
		T=\kappa_9\max\left\lbrace\frac{(f(x_0)-f^*){\cT_{0}}}{\cF_0},\frac{(f(x_0)-f^*)}{\eta \ep^2}\right\rbrace
		=\tilde{\mathcal{O}}\left(\frac{\log\left(\frac{1}{\delta}\right)}{\ep^{2}}\right).\numberthis\label{eq:Tchoicezerowosgc}
		\end{align*}
		Here, $\kappa_i, i=1,2,\cdots,9$ are absolute constants, and ${\cT_{0}}=\kappa_3\frac{\lge^2\log\left(d\right)^2}{\sqrt{\ep}}$ and $\cF_0=\kappa_8\ep^{1.5}$. Hence, the total number of zeroth-order oracle calls is $Tn_1=\tilde{\mathcal{O}}\left(\frac{d^{1.5}\sigma}{\ep^{5.5}}\right)$.
\end{theorem}
\begin{remark}
A generic reduction was proposed in~\cite{allen2018neon2} for using any algorithm that converges to a first-order stationary points at a particular rate, to converge to a local minimizer at the same rate. The results in~\cite{allen2018neon2} are not directly applicable to the zeroth-order setting due to their assumptions. However, assuming that their assumption could be relaxed to get it work in the zeroth-order setting, it is interesting to examine if the results in~\cite{ghadimi2013stochastic} for converging to first-order stationary solution could be combined with the reduction proposed in~\cite{allen2018neon2} to establish a result similar to Theorem~\ref{th:pgdmaintheoremwosgc}. To make the result of~\cite{ghadimi2013stochastic} hold with the same probability as in Theorem~\ref{th:pgdmaintheoremwosgc}, we would require $O(d~\epsilon^{-6})$ calls to the stochastic zeroth-order oracle. Hence, in certain regimes it is plausible we obtain improved results. It is interesting future work to examine this further rigorously.
\end{remark}

\section{Stochastic Cubic-Regularized Newton's Method}
\begin{algorithm*}[t]
	\caption{Cubic-Regularized Newton Algorithm }\label{alg:nccrn}	
	{\bf{Input:}} $x_1\in \mathbb{R}^d$, $T$, $M$, $n_1$, $n_2$\\
	{\bf for} $t=1$ to $T$ \bf do	\\
	{\bf Set} $g_t=\frac{1}{n_1}\sum_{i=1}^{n_1}g_{t,i}$ where
	\begin{align*}
	&g_{t,i}=\nabla F\left(x_t,\xi_{t,i}^G\right)  \tag*{(Higher-order)}\\
	&g_{t,i}=\frac{F(x_t+\nu u_{t,i}^G,\xi_{t,i}^G)-F(x_t,\xi_{t,i}^G)}{\nu}u_i^G \tag*{(Zeroth-order)}
	\end{align*}
	{\bf Set} $H_t=\frac{1}{n_2}\sum_{i=1}^{n_2}H_{t,i}$ where
	\begin{align*}
	&H_{t,i}=\nabla^2 F\left(x_t,\xi_{t,i}^H\right)  \tag*{(Higher-order)}\\
	&H_{t,i}=\frac{F(x_t+\nu u_{t,i}^H,\xi_{t,i}^H)+F(x_t-\nu u_{t,i}^H,\xi_{t,i}^H) -2F(x_t,\xi_{t,i}^H)}{2\nu^2}\left(u_{t,i}^H{u_{t,i}^H}^\top-I\right) \tag*{(Zeroth-order)}
	\end{align*}
	where $u_{t,i}^{G[H]}\sim \mathcal{N}\left(\pmb {0},\pmb{I}_d\right)$ $\forall t=1,2,\cdots,T, i=1,2,\cdots,n_1[n_2]$\\
	{\bf Update} 
	\begin{align*}
		x_{t+1}=\argmin_{y}m_t\left( x_t,y,g_t,H_t,M\right), \numberthis\label{eq:subprob}
	\end{align*}
	{where}
	\begin{align*}
	m_t(y)=f(x_t)+(y-x_t)^\top g_t+\frac{1}{2}(y-x_t)^\top H_t(y-x_t)
	+\frac{M}{6}\|y-x_t\|^3 \numberthis \label{eq:envelope}
	\end{align*}
	{\bf end for}
\end{algorithm*}
In this section we analyze Cubic-Regularized (CR) Newton method under interpolation regime. In non-interpolation like stochastic setting, CR Newton achieves a rate of $\order(\ep^{-3.5})$ as compared to $\order(\ep^{-4})$ attained by PSGD. Here we show that CR Newton achieves a rate of $\order(\ep^{-2.5})$ under SGC. Even though this rate is better than non-interpolation like stochastic setting, quite interestingly, CR Newton method fails to achieve deterministic rate of $\order(\ep^{-1.5})$ unlike PSGD. We believe that without stronger assumption on the Hessian estimator noise as well, CR Newton will perform worse than PSGD. 
In this section let $\calF_t$ be the filtration generated until time $t$, i.e., in the higher-order setting $\calF_t=\sigma(\{\xi_{i,j}^G\}_{i,j=1}^{t,n_1},\{\xi_{i,j}^H\}_{i,j=1}^{t,n_2})$, and in the zeroth-order setting $\calF_t=\sigma(\{\xi_{i,j}^G\}_{i,j=1}^{t,n_1},\{u_{i,j}^G\}_{i,j=1}^{t,n_1},\{\xi_{i,j}^H\}_{i,j=1}^{t,n_2},\{u_{i,j}^H\}_{i,j=1}^{t,n_2})$. We now present our main result.
\begin{theorem}\label{th:crnmaintheorem}
	Let $f$ be a function for which Assumptions~\ref{as:lipgrad}, and \ref{as:liphess} are true. Then under SGC, i.e., under Assumption~\ref{as:SGC}, for Algorithm~\ref{alg:nccrn}, we have:
\begin{enumerate}[label=\alph*), leftmargin=10pt,noitemsep]
		\item In the higher-order setting, choosing
		\begin{align*}
		T=\frac{144\left(f(x_1)-f^*\right)}{M\epsilon^\frac{3}{2}} \,
		n_1=\frac{\mu_0(\rho-1)}{\epsilon} \, n_2=\epsilon^{-1}, 
		M=\max\left(L_H,\frac{1}{4},\left(0.004L_G\epsilon^\frac{1}{4}+\sigma_2\epsilon^\frac{1}{4}\right),40\sigma_2\right) \numberthis\label{eq:crnewtonparameterchoice}
		\end{align*}
		we get,
$\max\left(\sqrt{\frac{\expec{\|\nabla f\left(x_R\right)\|}}{144M}},-\frac{\expec{\lambda_{1,R}}}{9M}\right)\leq \sqrt{\epsilon} $, where $\mu_0$ is a constant independent of $\epsilon$ and $d$, and $R$ is an integer random variable uniformly distributed over the support $\lbrace 1,2,\cdots,T\rbrace$. The total number of first-order and second-order oracle calls are hence $\order\left({\epsilon^{-\frac{5}{2}}}\right)$.
		\item In the zeroth-order setting, choosing
		\begin{align*}
		T=\frac{\mu_0\left(f(x_1)-f^*\right)}{M\epsilon^\frac{3}{2}},
		n_1=\frac{\mu_1(d+5)}{\epsilon},  M=\mu_4, 
		 \nu=\frac{\mu_3\ep}{(d+16)^\frac{5}{2}},  n_2=\frac{\mu_2(1+2\log 2d)(d+16)^4}{\epsilon} \numberthis\label{eq:crnewtonparameterchoicezero}
		\end{align*}
		we get, $\max\left(\sqrt{\expec{\|\nabla f\left(x_R\right)\|}},-\expec{\lambda_{1,R}}\right)\leq \order\left(\sqrt{\epsilon}\right)$, where $\mu_0,\mu_1,\mu_2,\mu_3,\mu_4$ are constants independent of $\ep$, and $d$, and $R$ is an integer random variable uniformly distributed over the support $\lbrace 1,2,\cdots,T\rbrace$.. The total number of first-order oracle calls is $\order\left({d}/{\epsilon^\frac{5}{2}}\right)$, and the number of second-order oracle calls is $\order\left({d^4\log d}/{\epsilon^\frac{5}{2}}\right)$.
	\end{enumerate}
\end{theorem}
 \vspace{-0.05in} \begin{remark}\label{scrnrateimprovement}
The above results only require Assumption~\ref{as:SGC}, which is a gradient-level property of interpolation condition. As SCRN is a second-order algorithm, an assumption like ``if the min eigenvalue of true Hessian at a point is non-negative, then min eigenvalue of stochastic Hessian is almost surely also non-negative'' might be required to capture second-order properties of interpolation. Such an assumption could then be used to obtain a result similar to Lemma 2.1 for stochastic Hessians, to improve the rates in Theorem~\ref{th:crnmaintheorem}. Formalizing this intuition is an extremely interesting future work.
  \end{remark}
\begin{remark}
In comparison to the PSGD algorithm, we obtain the results for the SCRN algorithm in expectation. We highlight that it is straightforward to obtain to obtain a high-probability result in the higher-order setting. However, it is technically challenging to do so for the zeroth-order setting. This is due to the difficulty associated with obtaining sharper concentration results for the zeroth-order Hessian estimator, which we leave as future work. In Theorem~\ref{th:crnmaintheorem}, we presented the results in expectation for both settings to maintain uniformity of presentation. In Algorithm~\ref{alg:nccrn} we assume that the exact solution to \eqref{eq:subprob} is available. We remark that it is possible to relax this assumption following the approach of~\cite{tripuraneni2018stochastic} which in turn leveraged the results in \cite{carmon2016gradient} showing that the subproblem in ~\eqref{eq:subprob} can be solved with high probability using gradient descent.
  \end{remark}

\section{Summary} 
In this work, we analyze the oracle complexity of two standard algorithms --the perturbed stochastic gradient descent algorithm and the stochastic cubic-regularized Newton's method--for escaping saddle-points in nonconvex stochastic optimization. We show that under interpolation-like conditions satisfied in modern over-parametrized machine learning problems, PSGD and SCRN obtain improved rates for escaping saddle-points. In particular the above stated improvements are obtained for the vanilla versions of PSGD and SCRN algorithms and are not based on any complicated variance reduction techniques. For future work, it is extremely interesting to bridge the gap between SCRN and its deterministic counterpart. The key to this is come up with a Hessian-based interpolation-like assumption, which is both practically meaningful and theoretically sound. 

\appendix



\section{Proof of Theorem~\ref{th:pgdmaintheorem}.}\label{sec:proofpgd}
\textbf{Preliminaries I:} We first present preliminary results regarding the zeroth-order setting.

\begin{lemma}\label{lm:zerogradconc}
	Let Assumption~\ref{as:lipgrad}, and \ref{as:subgauss} be true for $F$. Then, in the zeroth-order setting, $\norm{\zeta_t}$ is a $2/3$-sub-exponential variable. i.e.,
	\begin{align*}
	&\mathbb{P}\left(\norm{\zeta_t}\geq\tau\right)\leq 4d\exp\left(-K_1\min\left[\left(\frac{\sqrt{n_1}\tau'}{\Upsilon_t\sqrt{d}}\right)^2,\left(\frac{n_1\tau'}{\Upsilon_t\sqrt{d}}\right)^{2/3}\right]\right),\numberthis\label{eq:zerogradconc}
	\end{align*}
	where $\tau'=\tau-\frac{\nu}{2}L_G(d+3)^\frac{3}{2}$, and $\Upsilon_t=\frac{\nu L_G(d+2)}{2}+c_0\sqrt{(\rho-1)(d+1)}\|\nabla_t\|$.
\end{lemma}
We will choose $n_1$ such that we have $(n_1\tau'/(\Upsilon_t\sqrt{d}))^{2/3}\leq (\sqrt{n_1}\tau'/(\Upsilon_t\sqrt{d}))^2$. So from now on we will only consider the heavier subexponential tail. 
\begin{lemma}\label{lm:expecexpbound}
	Let Assumption~\ref{as:lipgrad}, and \ref{as:subgauss} be true for $F$. Then, in the zeroth-order setting
	\begin{align*}
	\expec{\exp(s(c_t\|\zeta_t\|)^\frac{1}{3})}\leq 9d\exp(s^2/b_{1,t}),
	\end{align*}
	where $s>0$, $b_{1,t}=b_{0,t}/c_t^{2/3}$, and $b_{0,t}=K_1n_1^{2/3}/(\Upsilon_t\sqrt{d})^{2/3}$. 
\end{lemma}
\begin{lemma}\label{lm:innerprodboundzero}
	Let Assumption~\ref{as:lipgrad}, and \ref{as:subgauss} be true for $F$. Then, in the zeroth-order setting for $l>0$, with probability at least $1-e^{-l}$ we have 
	\begin{align*}
	\eta\sum_{i=0}^{t-1}\nabla_i^\top\zeta_i\leq \frac{8\eta\sqrt{dt}}{K_1^\frac{3}{2}n_1}(t\log 9d+l)^\frac{3}{2}\sum_{i=0}^{t-1}\left(\frac{\nu L_G(d+2)}{2}\|\nabla_i\|+C_0\sqrt{(\rho-1)(d+1)}\|\nabla_i\|^2\right).
	\end{align*}
\end{lemma}
\begin{lemma}\label{lm:normsumboundzero}
	Let Assumption~\ref{as:lipgrad}, and \ref{as:subgauss} be true for $F$. Then, for $l>0$, with probability at least $1-e^{-l}$ we have 
	\begin{align*}
	\sum_{i=0}^{t-1}\|\zeta_i\|^2\leq \frac{128dt^2(t\log 9d +l)^3}{K_1^3n_1^2}\sum_{i=0}^{t-1}\left(\left(\frac{\nu L_G(d+2)}{2}\right)^2+C_0^2(\rho-1)(d+1)\|\nabla_i\|^2\right).
	\end{align*}
\end{lemma}
\textbf{Preliminaries II:} We next present preliminary results regarding the iterates of PSGD. First, we show that the effect of PSGD updates comprises of two parts - the first term on the RHS of \eqref{eq:descentfirst}, and \eqref{eq:descentzero} represent the decrease in the function values, and the rest of the terms on the RHS represent possible increase in function value due to noise in the gradient estimator and introduced perturbation. 
\begin{lemma}\label{lm:descent}
Under Assumption~\ref{as:lip}, \ref{as:lipgrad}, \ref{as:liphess},  \ref{as:SGC} and \ref{as:subgauss}, for any fixed $\cT_0,\cT_1,l>\log 4$, with probability at least $1-4e^{-l}$, for Algorithm~\ref{alg:pgd} we get
\begin{enumerate}[label=\alph*)]
\item for the first-order setting, choosing 
\begin{align*}
    n_1\geq 512lc(\rho-1) \quad \eta\leq \frac{32lc}{3L_G(l+c)}\numberthis\label{eq:descentparchoicefirst}
\end{align*}
we have
\begin{align*}
    f(x_{\cT_1})-f(x_0)\leq -\frac{\eta}{16}\sum_{i=0}^{\cT_1}\|\nabla_i\|^2+3c\eta^2r^2(\cT_{1}+l)L_G+32cl\eta r^2 \numberthis\label{eq:descentfirst}
\end{align*}
\item for the zeroth-order case, selecting parameters such that
\begin{align*}
    \frac{384L_GC_0^2d(\rho-1)(d+1){\cT_{0}}^2({\cT_{0}}\log 9d +l)^3}{K_1^3n_1^2}\leq\frac{1}{16}\numberthis\label{eq:parchoosecond1}\\
    \frac{8C_0\sqrt{(\rho-1)d(d+1){\cT_{0}}}}{K_1^\frac{3}{2}n_1}({\cT_{0}}\log 9d+l)^\frac{3}{2}\leq \frac{1}{16}\numberthis\label{eq:parchoosecond2}
\end{align*}
we have
\begin{align*}
    &f(x_{{\cT_{0}}})-f(x_0)\leq  -\frac{\eta}{16}\sum_{i=0}^{{{\cT_{0}}}-1}\|\nabla_i\|^2+\wp(r,l,\nu,\eta,d,{\cT_{0}})\numberthis\label{eq:descentzero}
\end{align*}
where
\begin{align*}
    &\wp(r,l,\nu,\eta,d,{\cT_{0}})=16cl\eta r^2+3cL_G\eta^2r^2({\cT_{0}}+l)\\
    +&\frac{8\nu \eta LL_G(d+2)\sqrt{d}{\cT_{0}}^\frac{3}{2}}{2K_1^\frac{3}{2}n_1}({\cT_{0}}\log 9d+l)^\frac{3}{2}
    +\frac{96L_G^3\nu^2\eta^2d(d+2)^2{\cT_{0}}^3({\cT_{0}}\log 9d +l)^3}{K_1^3n_1^2} 
\end{align*}
\end{enumerate}
\end{lemma}
In the following Lemma we show that when the function descent is small the iterates move only in a small region. 
\begin{lemma}\label{lm:xtx0distnormbound}
Under conditions of Lemma~\ref{lm:descent}, Algorithm~\ref{alg:pgd} satisfies
\begin{enumerate}[label=\alph*)]
    \item  for first-order setting, with probability at least $1-8d\cT_1e^{-l}$, for all $\tau\leq {\cT_{1}}$
\begin{align*}
    \|x_\tau-x_0\|^2\leq 32\eta \left({\cT_{1}}+2cl\frac{\rho-1}{n_1}\right)\left(f(x_0)-f(x_{\cT_{1}})+3c\eta^2r^2({\cT_{1}}+l)L_G+32cl\eta r^2\right) +4cl{\cT_{1}}\eta^2r^2 \numberthis\label{eq:xtx0distnormboundfirst}
\end{align*}
\item for zeroth-order setting, with probability at least $1-3d\cT_0e^{-l}$, for all $\tau\leq {\cT_{0}}$
\begin{align*}
 \|x_\tau-x_0\|^2\leq &  \eta {\cT_{0}}\left(32+\frac{16}{3L_G}\right)\left(f(x_0)-f(x_{\cT_{0}})+\wp(r,l,\nu,\eta,d,{\cT_{0}})\right)+4cl{\cT_{0}}\eta^2r^2\\
 +&\frac{L_G\eta^2{\cT_{0}}^2\nu^2(d+2)^2}{48C_0^2(\rho-1)(d+1)} \numberthis\label{eq:xtx0distnormboundzero}
\end{align*}
\end{enumerate}
\end{lemma}
We also require the following definition from~\cite{jin2019nonconvex}, to proceed. 
\begin{definition}\label{def:coupling}\cite{jin2019nonconvex}
Let $e_1$ be the eigen-vector corresponding to the minimum eigen-value of $\cH=\nabla^2 f(x_0)$, and $\gamma\vcentcolon=\lambda_{min}(\nabla^2 f(x_0))$. Also let $\calP_{-1}$ be the projection on to the complement subspace of $e_1$. Consider sequences ${x_t}$, and ${x_t'}$ that are obtained as separate versions of Algorithm~\ref{alg:pgd}, both starting from $x_0$. They are coupled in the first-order (zero-order) setting if both sequences are generated by the same $\calP_{-1}\theta_\tau$, and $\xi_\tau$ $\left(\{\xi_\tau,\{u_{\tau,i}\}_{i=1}^{n_1}\}\right)$, while in $e_1$ direction we have $e_1^\top\theta_\tau=-e_1^\top\theta_\tau'$.
\end{definition}
We next state some intermediate results in Lemma~\ref{lm:xthatsplit}--\ref{lm:qhqsgbound}, to prove in Lemma~\ref{lm:improvorloc} that starting from a saddle-point PSGD should either descend or the iterates will be stuck around the saddle point. Then in Lemma~\ref{lm:escapesaddle} we will show that the stuck region is narrow enough so that the iterates will escape and consequently the function will have sufficient descent. 
\begin{lemma}\cite{jin2019nonconvex}\label{lm:xthatsplit}
Consider the coupling sequences ${x_\tau}$ and ${x_\tau'}$ as in Definition~\ref{def:coupling} and let $\hx_\tau=x_\tau-x_\tau'$. Then $\hx_t=-q_h(t)-q_{sg}(t)-q_p(t)$, where:
\begin{align*}
    q_h(t)\vcentcolon=\eta\sum_{\tau=0}^{t-1}(I-\eta \cH)^{t-1-\tau}\Delta_\tau\hx_\tau, \quad q_{sg}(t)\vcentcolon=\eta\sum_{\tau=0}^{t-1}(I-\eta \cH)^{t-1-\tau}\hzt_\tau, \quad q_{p}(t)\vcentcolon=\eta\sum_{\tau=0}^{t-1}(I-\eta \cH)^{t-1-\tau}\hat{\theta}_\tau
\end{align*}
where $\Delta_t\vcentcolon=\int_0^1(\nabla^2f(\phi x_t+(1-\phi)x_t')d\phi-\cH$, and $\hzt_\tau\vcentcolon=\zeta_\tau-\zeta_\tau'$, $\hat{\theta}_\tau=\theta_\tau-\theta_\tau'$.
\end{lemma}
\begin{lemma}\cite{jin2019nonconvex}\label{lm:alphabetarelation}
Denote $\alpha(t)\vcentcolon=\left[\sum_{i=0}^{t-1}(1+\eta\gamma)^{2(t-1-\tau)}\right]^\frac{1}{2}$, and $\beta(t)=(1+\eta\gamma)^t/\sqrt{2\eta\gamma}$. If $\eta\gamma\in[0,1]$, then (1)$\alpha(t)\beta(t)$ for any $t\in\mathbb{N}$; and (2) $\alpha(t)\geq \beta(t)/\sqrt{3}$ for $t\geq \ln(2)/(\eta\gamma)$.
\end{lemma}
\begin{lemma}\cite{jin2019nonconvex}\label{lm:qptaubound}
Under the notation of Lemma~\ref{lm:xthatsplit}, and \ref{lm:alphabetarelation}, we have $\forall t>0$:
\begin{align*}
    &\mathbb{P}\left(\|q_p(t)\|\leq \frac{c\beta(t)\eta r}{\sqrt{d}}\sqrt{l}\right)\geq 1-2e^{-l}\\
    &\mathbb{P}\left(\|q_p({\cT_{1[0]}})\|\geq \frac{\beta({\cT_{1[0]}})\eta r}{10\sqrt{d}}\right)\geq\frac{2}{3}
\end{align*}
We use $1[0]$ to denote that the inequality holds for both subscripts $1$ and $0$.
\end{lemma}
\begin{lemma}\label{lm:qhqsgbound}
    Under the notation of Lemma~\ref{lm:xthatsplit} and \ref{lm:alphabetarelation}, if 
    \begin{align*}
        \eta\cS{\cT_{1[0]}}\max(L_H,L_G)\leq \frac{1}{l}  \quad c\leq \sqrt{l}/40 \numberthis\label{eq:parchoosecond4}
    \end{align*}
    \begin{enumerate}[label=\alph*)]
        \item \cite{jin2019nonconvex}then in the first-order case, we have
    \begin{align*}
        &~~\mathcal{P}\left(\min\{f(x_{\cT_{1}})-f(x_0),f(x_{\cT_{1}}')-f(x_0)\}\leq -\cF_1, \ or \ \forall t\leq {\cT_{1}}\vcentcolon \|q_h(t)+q_{sg}(t)\|\leq \frac{\beta(t)\eta t}{20\sqrt{d}}\right)\\
        \geq &~~1- 10d{\cT_{1}}^2\log\left(\frac{\cS_1\sqrt{d}}{\eta r}\right)e^{-l}
    \end{align*}
        \item in the zeroth-order case, we have
        \begin{align*}
        &~~\mathcal{P}\left(\min\{f(x_{\cT_{0}})-f(x_0),f(x_{\cT_{0}}')-f(x_0)\}\leq -\cF_0, \ or \ \forall t\leq {\cT_{0}}\vcentcolon \|q_h(t)+q_{sg}(t)\|\leq \frac{\beta(t)\eta t}{20\sqrt{d}}\right)\\
        \geq &~~1- 3{\cT_{0}}^2e^{-l}
    \end{align*}
    \end{enumerate}
\end{lemma}
\begin{lemma}\label{lm:improvorloc}
\begin{enumerate}[label=\alph*)]
    \item \cite{jin2019nonconvex} Under the setting of Lemma~\ref{lm:descent}, for the first-order setting, we have
    \begin{align*}
    &~~\mathbb{P}\left(\min\{f(x_{\cT_{1}})-f(x_0),f(x_{\cT_{1}}')-f(x_0)\}\leq -\cF_{1}, \ \text{or} \ \forall t\leq {\cT_{1}}: \max\{\|x_t-x_0\|^2,\|x_t'-x_0\|^2\}\leq \cS_{1}^2\right)\\
    \geq&~~1-16d{\cT_{1}} e^{-l}
\end{align*}
\item for the zeroth-order setting, we have
\begin{align*}
    &~~\mathbb{P}\left(\min\{f(x_{\cT_{0}})-f(x_0),f(x_{\cT_{0}}')-f(x_0)\}\leq -\cF_{0}, \ \text{or} \ \forall t\leq {\cT_{0}}: \max\{\|x_t-x_0\|^2,\|x_t'-x_0\|^2\}\leq \cS_{0}^2\right)\\
    \geq&~~1-4d{\cT_{0}} e^{-l}
\end{align*}
\end{enumerate}
\end{lemma}
In the following Lemma we show that while escaping from a saddle point, the PSGD descends more than it ascends with high probability.  
\begin{lemma} \label{lm:escapesaddle}
Let Under Assumption~\ref{as:lip}, \ref{as:lipgrad}, \ref{as:liphess},  \ref{as:SGC}, and \ref{as:subgauss} are true. Under condition \eqref{eq:parchoosecond4}, for any fixed $t_0>0$, let $x_0$ satisfies $$\|\nabla_0\|\leq \ep \quad \lambda_{min}(\nabla^2 f(x_0))\leq -\sqrt{L_H\ep}.$$ Then
\begin{enumerate}[label=\alph*)]
    \item  if $\eta, r, n_1$ are chosen as in \eqref{eq:allparameterchoicefirst}, ${\cT_{1}}={0.5\lge^3}/{\sqrt{\ep}}$, $ \cF_1={\ep^{1.5}}/{\lge^7}$, $\cS_1=\frac{\sqrt{\ep}}{\lge^2}$, $l=a_0\log\left(\frac{f(x_0)-f^*}{\delta\ep}\right)$, then the sequence generated by Algorithm~\ref{alg:pgd} in the first-order case satisfies
\begin{align*}
    &\mathbb{P}\left(f(x_{t_0+{\cT_{1}}})-f(x_{t_0})\leq 0.1\cF_1\right)\geq 1-4e^{-l} \numberthis\label{eq:smallascentfirst}\qquad \text{and}\\
    &\mathbb{P}\left(f(x_{t_0+{\cT_{1}}})-f(x_{t_0})\leq -\cF_1\right)\geq \frac{1}{3}-9d{\cT_{1}}^2\log\left(\frac{\cS_1\sqrt{d}}{\eta r}\right)e^{-l}\numberthis\label{eq:largedescentfirst}
\end{align*}
    \item if $\eta, r, n_1$ are chosen as in \eqref{eq:allparameterchoicezero}, ${\cT_{0}}=\kappa_3\frac{\lge^2\log\left(d\right)^2}{\sqrt{\ep}}$, $\cF_0=\kappa_8\ep^{1.5}$, $\cS_0=\frac{\kappa_7\sqrt{\ep}}{\lge^2}$ and $l=\kappa_6\log\left(\frac{d(f(x_0)-f^*)}{\delta\ep}\right)$, then the sequence generated by Algorithm~\ref{alg:pgd} in the zeroth-order case satisfies
    \begin{align*}
    &\mathbb{P}\left(f(x_{t_0+{\cT_{0}}})-f(x_{t_0})\leq 0.1\cF_0\right)\geq 1-4e^{-l}\numberthis\label{eq:smallascentzero} \qquad \text{and}\\
    &\mathbb{P}\left(f(x_{t_0+{\cT_{0}}})-f(x_{t_0})\leq -\cF_0\right)\geq \frac{1}{3}-\frac{3}{2}{\cT_{0}}^2e^{-l}\numberthis\label{eq:largedescentzero}
\end{align*}
\end{enumerate}
\end{lemma}
\textbf{Finishing the proof:} By combining the above results, we prove Theorem~\ref{th:pgdmaintheorem}. The proof is divided in two parts -- in the first part we show that the function descends enough when the gradient is large and in the second part we show that the iterates do escape from the saddle points and then function has sufficient descent. 

\textbf{Choice of parameters for Zeroth-order case.} As the expressions involved in the analysis of the zeroth order case are little complicated, we show explicitly here how to choose the parameters. First define, 
\begin{align}\label{textxidef}
\textXi:=\frac{32\sqrt{d({\cT_{0}}+1)}\eta \beta({\cT_{0}}+1)(({\cT_{0}}+1)\log 9d+\log 2+l)^\frac{3}{2}}{K_1^{3/2}n_1}
\end{align}
The choice of the parameters should be such that the following equations are satisfied:
\begin{align*}
&\frac{384L_GC_0^2d(\rho-1)(d+1){\cT_{0}}^2({\cT_{0}}\log 9d +l)^3}{K_1^3n_1^2}\leq\frac{1}{16},\\
&\frac{8C_0\sqrt{(\rho-1)d(d+1){\cT_{0}}}}{K_1^\frac{3}{2}n_1}({\cT_{0}}\log 9d+l)^\frac{3}{2}\leq \frac{1}{16},\\
&\eta\cS_0{\cT_{0}}\max(L_H,L_G)\leq \frac{1}{l},  \quad
 c\leq \sqrt{l}/40,\\
&\textXi \cdot \sum_{i=0}^{{\cT_{0}}}\left(\frac{\nu L_G(d+2)}{2}+C_0\sqrt{(\rho-1)(d+1)}L\right)\leq \frac{\beta({\cT_{0}})r}{40\sqrt{d}},\\
&\frac{(1+\eta \gamma)^{\cT_{0}}\sqrt{\eta} r}{40\sqrt{2\gamma d}}>\cS_0 ,\quad
\wp(r,l,\nu,\eta,d,\cT_0)\leq 0.1\cF_0.
\end{align*}
Furthermore, we need to ensure the RHS of \eqref{eq:xtx0distnormboundzero} is of the same order of $\cS_0^2$.

\begin{proof}[Proof of Theorem~\ref{th:pgdmaintheorem}]
	\begin{enumerate}[label=\alph*)]
		\item	
		\begin{enumerate}[label=\arabic*.]
			\item First we look at the time instants where $\|\nabla_t\|\geq \ep$. If there are more than $\frac{T}{4}$ such time steps, then using Lemma~\ref{lm:descent} we have, with probability at least $1-4e^{-l}$
			\begin{align*}
			f(x_T)-f(x_0)&\leq -\frac{T\ep^2}{64\lge^2}+ 3cL_G\frac{\ep^3}{\lge^{10}}\left(\frac{0.5\lge^3}{\sqrt{\ep}}+\lge\right)+32c\frac{\ep^3}{\lge^7}\\
			&\leq -\frac{T\ep^2}{128\lge^2}
			\end{align*}
			Letting $T$ as in \eqref{eq:Tchoicefirst}, we get $f(x_T)\leq f(x_0)-T\ep^2/128\lge^2<f^*$ which is impossible. 
			\item As follows from Claim 2 in the proof of Theorem 16 of \cite{jin2019nonconvex}, we have, with probability at least $1-10d{\cT_{0}}^2T^2\log(\cS_1\sqrt{d}/(\eta r))e^{-l}$
			\begin{align*}
			f(x_T)-f(x_0)\leq -0.1\frac{T\cF_1}{{\cT_{1}}} 
			\end{align*}
			which implies $f(x_T)\leq f(x_0)-0.1T\cF_1/{\cT_{1}}<f^*$ which is impossible. 
		\end{enumerate}
		\item
		\begin{enumerate}[label=\arabic*.]
			\item First we look at the time instants where $\|\nabla_t\|\geq \ep$. If the parameters are chosen as in \eqref{eq:allparameterchoicezero}, ${\cT_{0}}=\kappa_3\frac{\lge^2\log\left(d\right)^2}{\sqrt{\ep}}$, and $l=\kappa_6\log\left(\frac{d(f(x_0)-f^*)}{\delta\ep}\right)$ then we have,
			\begin{align*}
			\wp(r,l,\nu,\eta,d,{\cT_{0}})=\order\left(\ep^{1.5}\right)
			\end{align*}
			If there are more than $\frac{T}{4}$ such time steps, then using Lemma~\ref{lm:descent} we have, with probability at least $1-4e^{-l}$
			\begin{align*}
			f(x_T)-f(x_0)\leq -\frac{\kappa_0T\ep^{2.5}}{64\lge}+ \order\left(\ep^{1.5}\right)
			\leq -\frac{\kappa_0T\ep^{2.5}}{128\lge}
			\end{align*}
			Letting $T$ as in \eqref{eq:Tchoicezero}, $\kappa_9\geq 128$, and $\kappa_0\kappa_3/\kappa_8\geq 128$ we get $f(x_T)\leq f(x_0)-\frac{\kappa_0T\ep^{2.5}}{128\lge}< f^*$ which is impossible. 
			\item As follows from Claim 2 in the proof of Theorem 16 of \cite{jin2019nonconvex}, we have, with probability at least $1-3{\cT_{0}}^2T^2e^{-l}$
			\begin{align*}
			f(x_T)-f(x_0)\leq -0.1\frac{T\cF_0}{{\cT_{0}}} 
			\end{align*}
			which implies $f(x_T)\leq f(x_0)-0.1T\cF_0/{\cT_{0}}<f^*$ when $\kappa_9\geq 128$ and $T$ is as in \eqref{eq:allparameterchoicezero}, which is impossible. 
		\end{enumerate}
	\end{enumerate}
\end{proof}

\begin{proof}[Proof of Theorem~\ref{th:pgdmaintheoremwosgc}]
	The proof of Theorem~\ref{th:pgdmaintheoremwosgc} is same as Theorem~\ref{th:pgdmaintheorem} except for the concentration properties of $\|\zeta_t\|$. In this case we have $\|\zeta_t\|$ to be $\alpha$-sub-exponential with coefficient $(\Upsilon_t\sqrt{d}/n_1)^{2/3}$ where $$\Upsilon_t=\frac{\nu L_G(d+2)}{2}+C_0(\sigma+\|\nabla f(x_t)\|)\sqrt{d+1}.$$ So there is an extra term $C_0\sigma\sqrt{d+1}$ which can neither be made smaller using $\nu$ nor is of the same order as $\nabla f(x_t)$ so that it can be subsumed in other terms involving $\nabla f(x_t)$. Hence, the only way to make the coefficient smaller, which is essential in the proof, is to increase $n_1$. This is main reason why the rate deteriorates in the absence if SGC. For the sake of completeness, we provide below the set of conditions that need to be satisfied to pick the parameters in this setting, below.

\textbf{Choice of parameters for Zeroth-order case when SGC does not hold.} When SGC does not hold in the zeroth-order setting the conditions to be satisfied are:
\begin{align*}
&\frac{384L_GC_0^2d(\rho-1)(d+1){\cT_{0}}^2({\cT_{0}}\log 9d +l)^3}{K_1^3n_1^2}\leq\frac{\ep^2}{16},\\
&\frac{8C_0\sqrt{(\rho-1)d(d+1){\cT_{0}}}}{K_1^\frac{3}{2}n_1}({\cT_{0}}\log 9d+l)^\frac{3}{2}\leq \frac{\ep}{16},\\
&\eta\cS_0{\cT_{0}}\max(L_H,L_G)\leq \frac{1}{l},  \quad
c\leq \sqrt{l}/40,\\
&\textXi \cdot \sum_{i=0}^{{\cT_{0}}}\left(\frac{\nu L_G(d+2)}{2}+C_0\sqrt{(\rho-1)(d+1)}L\right)\leq \frac{\beta({\cT_{0}})r}{40\sqrt{d}},\\
&\frac{(1+\eta \gamma)^{\cT_{0}}\sqrt{\eta} r}{40\sqrt{2\gamma d}}>\cS_0 ,\quad
\wp(r,l,\nu,\eta,d,\cT_0)\leq 0.1\cF_0.
\end{align*}
Furthermore, we need to ensure the RHS of \eqref{eq:xtx0distnormboundzero} is of the same order of $\cS_0^2$.
\end{proof}

\subsection{Proofs of Lemmas related to Perturbed Stochastic Gradient Descent}
\begin{assumption}\cite{jin2019nonconvex}\label{as:subgauss1}
	Consider random vectors $X_1,X_2,\cdots,X_n \in \mathbb{R}^d$, and the corresponding filtrations $\calF_i=\sigma(X_1,X_2,\cdots,X_i)$ for $i=1,2,\cdots,n$, such that $X_i|\calF_{i-1}$ is zero-mean nSG$(\sigma_i)$ with $\sigma_i\in \calF_{i-1}$. That is,
	\begin{align*}
	\expec{X_i|\calF_{i-1}}=0, \quad P(\|X_i\|\geq t|\calF_{i-1})\leq e^{-{\frac{t^2}{2\sigma_i^2}}}, \quad \forall t\in \mathbb{R},\forall i=1,2,\cdots,n.
	\end{align*}
\end{assumption}
\begin{lemma}\cite{jin2019nonconvex} \label{lm:innerprodsumbound}
	Let $X_1,X_2,\cdots,X_n \in \mathbb{R}^d$ satisfy Assumption~\ref{as:subgauss}. $u_i\in \calF_{i-1}$ be a random vector for $i=1,2,\cdots,n$. Then for any $l>0,\ \lambda>0$, there exists absolute constant $c$ such that, with probability at least $1-e^{-l}$:
	\begin{align*}
	\sum_i u_i^\top X_i\leq c\lambda\sum_i \|u_i\|^2\sigma_i^2+\frac{l}{\lambda}
	\end{align*}
\end{lemma}
\begin{lemma}\cite{jin2019nonconvex} \label{lm:normsumbound}
	Let $X_1,X_2,\cdots,X_n \in \mathbb{R}^d$ satisfy Assumption~\ref{as:subgauss} with $\sigma_1=\sigma_2=\cdots=\sigma_n=\sigma$. Then for any $l>0,\ \lambda>0$, there exists absolute constant $c$ such that, with probability at least $1-e^{-l}$:
	\begin{align*}
	\sum_i \| X_i\|^2\leq c\sigma^2(n+l)
	\end{align*}
\end{lemma}
\begin{lemma}\cite{jin2019nonconvex} \label{lm:sumnormbound}
	Let $X_1,X_2,\cdots,X_n \in \mathbb{R}^d$ satisfy Assumption~\ref{as:subgauss} with fixed $\{\sigma_i\}$ then for any $l>0$, there exists an aboslute coonstant $c$ such that, with probability at least $1-2de^{-l}$:
	\begin{align*}
	\|\sum_{i=1}^nX_i\|\leq c \sqrt{\sum_{i=1}^n\sigma_i^2l}
	\end{align*}
\end{lemma}
Let $F_{\nu}(x,\xi)=\mathbf{E}_u\left[F(x+\nu u,\xi)\right]$, and $g_{t,i}^j$, and $\nabla F_\nu(x_t,\xi_i)^j$ denote the $j$-th coordinate of the vector $g_{t,i}=\frac{F(x_t+\nu u_i,\xi_i)-F(x_t,\xi_i)}{\nu}u_i$, and $\nabla F_\nu(x_t,\xi_i)$ respectively.
\begin{lemma}\cite{nesterov2017random}\label{lm:gradnutostochgraderror} Let Assumption~\ref{as:lipgrad} be true for $F$. Then
	\begin{align*}
	\|\nabla F_\nu(x,\xi)-\nabla F(x,\xi)\|\leq \frac{\nu}{2}L_G(d+3)^\frac{3}{2}
	\end{align*}
\end{lemma}
\begin{lemma}\label{lm:gaussnormbound}\cite{nesterov2017random} 
	For a Gaussian random vector $u\sim N(0,I_d)$, we have
	\begin{align*}
	\expec{\|u\|^k}\leq (d+k)^\frac{k}{2}
	\end{align*}
\end{lemma}
\begin{lemma}\cite{shen2019nonasymptotic}\label{lm:prodconcentration}
	Let $(X_i,Y_i)$, $i = 1,2,\cdots,n$ be $n$ independent copies of random variables $X$ and $Y$. Let $X$ be a sub-Gaussian random variable with sub-gaussian norm $\|X\|_{\psi_2} \leq \Upsilon_1$, and $Y$ be a sub-exponential random variable with sub-exponential norm $\|Y\|_{\psi_1} \leq \Upsilon_2$ for some constants $\Upsilon_1$ and $\Upsilon_2$. Then for any $t\geq K\max\left(\Upsilon_1,\Upsilon_1^3\right)\Upsilon_2$
	we have
	\begin{align*}
	\mathbb{P}\left(\abs{\sum_{i=1}^{n}{X_iY_i}-\expec{XY}}\geq t\right)\leq 4\exp\left(-K_1\min\left[\left(\frac{t}{\sqrt{n}\Upsilon_1\Upsilon_2}\right)^2,\left(\frac{t}{\Upsilon_1\Upsilon_2}\right)^{2/3}\right]\right)
	\end{align*}
	where $K$ and $K_1$ are absolute constants.
\end{lemma}
\begin{proof}[Proof of Lemma~\ref{lm:zerogradconc}]
	Let us write $g_{t,i}=\phi(\nu,u_i,\xi_i)u_i$ where $\phi(\nu,u_i,\xi_i)=\frac{F(x_t+\nu u_i,\xi_i)-F(x_t,\xi_i)}{\nu}$.
	We will show that $\phi(\nu,u_i,\xi_i)$ is a sub-exponential random variable by showing that its sub-exponential norm or $\psi_1$-norm, defined as $\|.\|_{\psi_1}=\sup_{p\geq 1}p^{-1}\expec{\abs{.}^p}^{p^{-1}}$, is finite. 
	\begin{align*}
	\|\phi(\nu,u_i,\xi_i)\|_{\psi_1}=\sup_{p\geq 1}\frac{1}{p}\expec{\abs{\phi(\nu,u_i,\xi_i)}^p}^\frac{1}{p}
	=\sup_{p\geq 1}\frac{1}{p}\mathbf{E}_{\xi_i}\left[\mathbf{E}_{u_i}\left[\abs{\phi(\nu,u_i,\xi_i)}^p\right]\right]^\frac{1}{p} \numberthis\label{eq:iteratedexpecsubexp}
	\end{align*}
	We first concentrate on the term $\mathbf{E}_{u_i}\left[\abs{\phi(\nu,u_i,\xi_i)}^p\right]$.
	\begin{align*}
	\mathbf{E}_{u_i}\left[\abs{\phi(\nu,u_i,\xi_i)}^p\right]=\mathbf{E}_{u_i}\left[\left\lvert\frac{F(x_t+\nu u_i,\xi_i)-F(x_t,\xi_i)-\nu\nabla F(x_t,\xi_i)^\top u_i}{\nu}+ \nabla F(x_t,\xi_i)^\top u_i\right\rvert^p\right]  
	\end{align*}
	By Minkowski's inequality, 
	\begin{align*}
	&\mathbf{E}_{u_i}\left[\abs{\phi(\nu,u_i,\xi_i)}^p\right]\\
	\leq& \left[\mathbf{E}_{u_i}\left[\left\lvert\frac{F(x_t+\nu u_i,\xi_i)-F(x_t,\xi_i)-\nu\nabla F(x_t,\xi_i)^\top u_i}{\nu}\right\rvert^p\right]^\frac{1}{p}+\mathbf{E}_{u_i}\left[\left\lvert \nabla F(x_t,\xi_i)^\top u_i\right\rvert^p\right]^\frac{1}{p}\right]^p\\
	\leq &\left[\frac{\nu L_G}{2}\mathbf{E}_{u_i}\left[\|u_i\|^{2p}\right]^\frac{1}{p}+\|\nabla F(x_t,\xi_i)\|\mathbf{E}_{u_i}\left[\|  u_i\|^p\right]^\frac{1}{p}\right]^p
	\end{align*}
	Using Lemma~\ref{lm:gaussnormbound},
	\begin{align*}
	\mathbf{E}_{u_i}\left[\abs{\phi(\nu,u_i,\xi_i)}^p\right]\leq \left[\frac{\nu  L_G(d+2p)}{2}+\sqrt{d+p}\|\nabla F(x_t,\xi_i)\|\right]^p
	\end{align*}
	Now from \eqref{eq:iteratedexpecsubexp}, using Minkowski's inequality, we get
	\begin{align*}
	&\|\phi(\nu,u_i,\xi_i)\|_{\psi_1}\leq \sup_{p\geq 1}\frac{1}{p}\mathbf{E}_{\xi_i}\left[\left(\frac{\nu L_G(d+2p)}{2}\right)^p\right]^\frac{1}{p}+\sup_{p\geq 1}\frac{1}{p}\mathbf{E}_{\xi_i}\left[\left(\sqrt{d+p}\|\nabla F(x_t,\xi_i)\|\right)^p\right]^\frac{1}{p}\\
	\leq& \frac{\nu L_G(d+2)}{2}+\sup_{p\geq 1}\sqrt{\frac{d+p}{p}}\frac{1}{\sqrt{p}}\mathbf{E}_{\xi_i}\left[\| \nabla F(x_t,\xi_i)\|^p\right]^\frac{1}{p}\\
	\leq & \frac{\nu L_G(d+2)}{2}+\sup_{p\geq 1}\left(\sqrt{\frac{d+p}{p}}\sup_{p\geq 1}\frac{1}{\sqrt{p}}\mathbf{E}_{\xi_i}\left[\| \nabla F(x_t,\xi_i)\|^p\right]^\frac{1}{p}\right)
	\end{align*}
	Now,
	\begin{align*}
	&\mathbf{E}_{\xi_i}\left[\|\nabla F(x_t,\xi_i)\|^p\right]^{p^{-1}}\\
	\leq &\mathbf{E}_{\xi_i}\left[\left(\|\nabla F(x_t,\xi_i)-\nabla f(x_t)+\nabla f(x_t)\|\right)^p\right]^{p^{-1}}\\
	\leq & \mathbf{E}_{\xi_i}\left[2^{p-1}\|\nabla F(x_t,\xi_i)-\nabla f(x_t)\|^p+2^{p-1}\|\nabla f(x_t)\|^p\right]^{p^{-1}}\\
	\leq & 2\mathbf{E}_{\xi_i}\left[\|\nabla F(x_t,\xi_i)-\nabla f(x_t)\|^p\right]^{p^{-1}}+2\|\nabla f(x_t)\|
	\end{align*}
	From \eqref{eq:sgsgc} we have, $\sup_{p\geq 1}p^{-1/2}\mathbf{E}_{\xi_i}\left[\left(\|\nabla F(x_t,\xi_i)-\nabla f(x_t)\|\right)^p\right]^{p^{-1}}\leq c_0'\sqrt{\rho-1}\|\nabla_t\|$ where $c_0$ is a constant. Then,
	\begin{align*}
	&\|\phi(\nu,u_i,\xi_i)\|_{\psi_1}\leq \frac{\nu L_G(d+2)}{2}+(2+c_0'\sqrt{(\rho-1)})\sqrt{d+1}\|\nabla_t\|
	\end{align*}
	We also have, $\|u_i^j\|_{\psi_2}\leq 1$, and $\expec{g_{t,i}}=\nabla f_\nu(x_t)$. Then using Lemma~\ref{lm:prodconcentration}, we have $\forall~j=1,2,\cdots,d$
	\begin{align*}
	\mathbb{P}\left(\frac{1}{n_1}\left\lvert\sum_{i=1}^{n_1}\left(g_{t,i}^j-\nabla f_\nu(x_t)^j\right)\right\rvert\geq \tau\right)\leq 4\exp\left(-K_1\min\left[\left(\frac{\sqrt{n_1}\tau}{\Upsilon_t}\right)^2,\left(\frac{n_1\tau}{\Upsilon_t}\right)^{2/3}\right]\right)
	\end{align*}
	where $\Upsilon_t=\frac{\nu L_G(d+2)}{2}+c_0\sqrt{(\rho-1)(d+1)}\|\nabla_t\|$.
	Using union bound,
	\begin{align*}
	&\mathbb{P}\left(\norm{\frac{1}{n_1}\sum_{i=1}^{n_1}g_{t,i}-\nabla f_\nu(x_t)}\geq\tau\right)\leq \mathbb{P}\left(\exists~j\in\{1,2,\cdots,d\} s.t. \ \left\lvert\frac{1}{n_1}\sum_{i=1}^{n_1}g_{t,i}^j-\nabla f_\nu(x_t)^j\right\rvert\geq\tau/\sqrt{d}\right)\\
	&\leq \sum_{j=1}^d \mathbb{P}\left(\left\lvert\frac{1}{n_1}\sum_{i=1}^{n_1} g_{t,i}^j-\nabla f_\nu(x_t)^j\right\rvert\geq\tau/\sqrt{d}\right)\leq 4d\exp\left(-K_1\min\left[\left(\frac{\sqrt{n_1}\tau}{\Upsilon_t\sqrt{d}}\right)^2,\left(\frac{n_1\tau}{\Upsilon_t\sqrt{d}}\right)^{2/3}\right]\right)
	\end{align*}
	Using Lemma~\ref{lm:gradnutostochgraderror} we have
	\begin{align*}
	&\mathbb{P}\left(\norm{\frac{1}{n_1}\sum_{i=1}^{n_1}g_{t,i}-\nabla f(x_t)}\geq\tau\right)\leq \mathbb{P}\left(\norm{\frac{1}{n_1}\sum_{i=1}^{n_1}g_{t,i}-\nabla_\nu f(x_t)}\geq\tau-\frac{\nu L_G(d+3)^\frac{3}{2}}{2}\right)
	\end{align*}
\end{proof}
\begin{proof}[Proof of Lemma~\ref{lm:expecexpbound}]
	\begin{align*}
	&\expec{(c_t\|\zeta_t\|)^\frac{k}{3}}=\int_{0}^{\infty}\mathbb{P}\left((c_t\|\zeta_t\|)^\frac{k}{3}>\tau\right)d\tau=\int_{0}^{\infty}\mathbb{P}\left(\|\zeta_t\|>\tau^\frac{3}{k}/c_t\right)d\tau\\
	\leq& \int_{0}^{\infty}4d\exp(-b_{1,t}\tau^{\prime 2/k})d\tau\leq \int_{-\frac{\nu L_G(d+3)^\frac{3}{2}}{2}}^{\infty}4d\exp(-b_{1,t}\tau^{2/k})d\tau\leq \int_{0}^{\infty}8d\exp(-b_{1,t}\tau^{2/k})d\tau
	\end{align*}
	Substituting, $u=b_{1,t}\tau^{2/k}$ we have,
	\begin{align*}
	\expec{(c_t\|\zeta_t\|)^\frac{k}{3}}\leq \int_{0}^{\infty}4dk{b}_{1,t}^{-k/2}e^{-u}u^{k/2-1}du=4dkb_{1,t}^{-k/2}\Gamma \left( k/2 \right)
	\end{align*}
	Using $2(k!)^2\leq (2k)!$, and $\Gamma(k+1/2)=(2k)!\sqrt{\pi}/(4^kk!)$, we have
	\begin{align*}
	&\expec{e^{s(c_t\|\zeta_t\|)^\frac{1}{3}}}=1+\sum_{k=1}^\infty\expec{\frac{s^k(c_t\|\zeta_t\|)^\frac{k}{3}}{k!}}\leq 1+\sum_{k=1}^\infty\frac{s^k}{k!}4dkb_{1,t}^{-k/2}\Gamma\left( k/2 \right)\\
	&\leq 1+4d\left[\sum_{k=1}^\infty\frac{2ks^{2k}b_{1,t}^{-k}}{(2k)!}\Gamma(k)+\sum_{k=0}^\infty\frac{(2k+1)s^{2k+1}b_{1,t}^{-k-1/2}}{(2k+1)!}\Gamma(k+1/2)\right] \\
	&\leq 1+4d\left[\sum_{k=1}^\infty\frac{s^{2k}b_{1,t}^{-k}}{k!}+\sqrt{\frac{\pi s^2}{b_{1,t}}}\sum_{k=0}^\infty\frac{s^{2k}b_{1,t}^{-k}}{4^kk!}\right]\\
	&\leq 1+4d\left[e^\frac{s^2}{b_{1,t}}+\sqrt{\frac{\pi s^2}{b_{1,t}}}e^{\frac{s^2}{4b_{1,t}}}\right]\leq 1+8de^{\frac{s^2}{b_{1,t}}}\leq 9de^\frac{s^2}{b_{1,t}}
	\end{align*}
\end{proof}
\begin{proof}[Proof of Lemma~\ref{lm:innerprodboundzero}]
	Setting $c=\eta\|\nabla_i\|$, using Lemma~\ref{lm:expecexpbound} we have
	\begin{align*}
	\expec{e^{s(\eta\|\nabla_i\|\|\zeta_i\|)^\frac{1}{3}}}\leq 9de^\frac{s^2}{b_{1,i}}.
	\end{align*}
Hence, we have the following:
	\begin{align*}
	&\expec{\exp\left(s\sum_{i=0}^{t-1}(\eta\|\nabla_i\|\|\zeta_i\|)^\frac{1}{3}-\sum_{i=0}^{t-1}\frac{s^2}{b_{1,i}}\right)}\\
	=&\expec{\exp\left(s\sum_{i=0}^{t-2}(\eta\|\nabla_i\|\|\zeta_i\|)^\frac{1}{3}-\sum_{i=0}^{t-1}\frac{s^2}{b_{1,i}}\right)\expec{\exp\left(s(\eta\|\nabla_{t-1}\|\|\zeta_{t-1}\|)^\frac{1}{3}\right)|\calF_{t-2}}}\\
	=&9d\expec{\exp\left(s\sum_{i=0}^{t-2}(\eta\|\nabla_i\|\|\zeta_i\|)^\frac{1}{3}-\sum_{i=0}^{t-1}\frac{s^2}{b_{1,i}}\right)e^\frac{s^2}{b_{1,t-1}}}\\
	=&9d\expec{\exp\left(s\sum_{i=0}^{t-2}(\eta\|\nabla_i\|\|\zeta_i\|)^\frac{1}{3}-\sum_{i=0}^{t-2}\frac{s^2}{b_{1,i}}\right)}.
	\end{align*}
	Continuing like above we get,
	\begin{align*}
	\expec{\exp\left(s\sum_{i=0}^{t-1}(\eta\|\nabla_i\|\|\zeta_i\|)^\frac{1}{3}-\sum_{i=0}^{t-1}\frac{s^2}{b_{1,i}}\right)} \leq (9d)^t. \numberthis\label{eq:martingalediffexpec}
	\end{align*}
Now, we attempt the main result. Note that, we have	
	\begin{align*}
	\mathbb{P}\left(\eta\sum_{i=0}^t\nabla_i^\top\zeta_i\geq \tau\right)&\leq \mathbb{P}\left(\eta\sum_{i=0}^t\|\nabla_i\|\|\zeta_i\|\geq \tau\right)\\
	&\leq \mathbb{P}\left(\sum_{i=0}^t(\eta\|\nabla_i\|\|\zeta_i\|)^\frac{1}{3}\geq \tau^\frac{1}{3}\right)\\
	 &=\mathbb{P}\left(s\sum_{i=0}^t(\eta\|\nabla_i\|\|\zeta_i\|)^\frac{1}{3}-\sum_{i=0}^{t-1}\frac{s^2}{b_{1,i}}\geq s\tau^\frac{1}{3}-\sum_{i=0}^{t-1}\frac{s^2}{b_{1,i}}\right)\\
	&=\mathbb{P}\left(\exp\left(s\sum_{i=0}^t(\eta\|\nabla_i\|\|\zeta_i\|)^\frac{1}{3}-\sum_{i=0}^{t-1}\frac{s^2}{b_{1,i}}\right)\geq \exp\left(s\tau^\frac{1}{3}-\sum_{i=0}^{t-1}\frac{s^2}{b_{1,i}}\right)\right)\\
	&\leq  \frac{\expec{\exp\left(s\sum_{i=0}^{t-1}(\eta\|\nabla_i\|\|\zeta_i\|)^\frac{1}{3}-\sum_{i=0}^{t-1}\frac{s^2}{b_{1,i}}\right)}}{\exp\left(s\tau^\frac{1}{3}-\sum_{i=0}^{t-1}\frac{s^2}{b_{1,i}}\right)}\\
	&\leq \exp\left(t\log 9d-s\tau^\frac{1}{3}+\sum_{i=0}^{t-1}\frac{s^2}{b_{1,i}}\right).
	\end{align*}
	The RHS is minimized at $s=\frac{\tau^{1/3}}{2\sum_{i=0}^{t-1}\frac{1}{b_{1,i}}}$. Substituting for $s$ this value, for some $l>0$ we have:
	$$t\log 9d-\tau^\frac{2}{3}/\left(4\sum_{i=0}^{t-1}\frac{1}{b_{1,i}}\right)=-l.$$ Hence, we have $$\tau=\left(4\sum_{i=0}^{t-1}\frac{1}{b_{1,i}}(t\log 9d +l)\right)^{3/2}.$$
Finally, to prove the statement of the Lemma, note that	
	\begin{align*}
	\left(\sum_{i=0}^{t-1}\frac{1}{b_{1,i}}\right)^\frac{3}{2}=\frac{\sqrt{d}}{K_1^\frac{3}{2}n_1}\left(\sum_{i=0}^{t-1}(c_i\Upsilon_i)^\frac{2}{3}\right)^\frac{3}{2}\leq \frac{\eta\sqrt{dt}}{K_1^\frac{3}{2}n_1}\sum_{i=0}^{t-1}\left(\frac{\nu L_G(d+2)}{2}\|\nabla_i\|+C_0\sqrt{(\rho-1)(d+1)}\|\nabla_i\|^2\right)
	\end{align*}
\end{proof}
\begin{proof}[Proof of Lemma~\ref{lm:normsumboundzero}]
	From \eqref{eq:martingalediffexpec} we have,
	\begin{align*}
	\expec{\exp\left(s\sum_{i=0}^{t-1}(\|\zeta_i\|)^\frac{1}{3}-\sum_{i=0}^{t-1}\frac{s^2}{b_{0,i}}\right)} \leq (9d)^t 
	\end{align*}
	where $b_{0,i}$ is as defined in Lemma~\ref{lm:expecexpbound}.
	\begin{align*}
	&\mathbb{P}\left(\sum_{i=0}^{t-1}\|\zeta_i\|^2\geq \tau\right)
	\leq \mathbb{P}\left(s\sum_{i=0}^{t-1}\|\zeta_i\|^\frac{1}{3}-\sum_{i=0}^{t-1}\frac{s^2}{b_{0,i}}\geq s\tau^\frac{1}{6}-\sum_{i=0}^{t-1}\frac{s^2}{b_{0,i}}\right)\\
	=&\mathbb{P}\left(\exp\left(s\sum_{i=0}^{t-1}\|\zeta_i\|^\frac{1}{3}-\sum_{i=0}^{t-1}\frac{s^2}{b_{0,i}}\right)\geq \exp\left(s\tau^\frac{1}{6}-\sum_{i=0}^{t-1}\frac{s^2}{b_{0,i}}\right)\right)\\
	\leq & \frac{\expec{\exp\left(s\sum_{i=0}^{t-1}\|\zeta_i\|^\frac{1}{3}-\sum_{i=0}^{t-1}\frac{s^2}{b_{0,i}}\right)}}{\exp\left(s\tau^\frac{1}{6}-\sum_{i=0}^{t-1}\frac{s^2}{b_{0,i}}\right)}\leq \exp\left(t\log 9d-s\tau^\frac{1}{6}+\sum_{i=0}^{t-1}\frac{s^2}{b_{0,i}}\right)
	\end{align*}
	Following steps as in Lemma~\ref{lm:innerprodboundzero} we have, $\tau=\left(4\sum_{i=0}^{t-1}\frac{1}{b_{0,i}}(t\log 9d +l)\right)^3$.
	\begin{align*}
	\left(\sum_{i=0}^{t-1}\frac{1}{b_{0,i}}\right)^3=\frac{d}{K_1^3n_1^2}\left(\sum_{i=0}^{t-1}\Upsilon_i^\frac{2}{3}\right)^3\leq \frac{2dt^2}{K_1^3n_1^2}\sum_{i=0}^{t-1}\left(\left(\frac{\nu L_G(d+2)}{2}\right)^2+C_0^2(\rho-1)(d+1)\|\nabla_i\|^2\right)
	\end{align*}
\end{proof}
\begin{proof}[Proof of Lemma~\ref{lm:descent}]
	\begin{enumerate}[label=\alph*)]
		\item 
		\begin{align*}
		f(x_{t+1})\leq& f(x_t)+\nabla_t^\top (x_{t+1}-x_t)+\frac{L_G}{2}\|x_{t+1}-x_t\|^2\\
		\leq & f(x_t)-\eta\nabla_t^\top (\nabla_t+\tilde{\zeta}_t)+\frac{\eta^2L_G}{2}\left(\frac{3}{2}\|\nabla_t\|^2+3\|\tilde{\zeta}_t\|^2\right)\\
		\leq & f(x_t)-\frac{\eta}{4}\|\nabla_t\|^2-\eta \nabla_t^\top \tilde{\zeta}_t+\frac{3\eta^2L_G}{2}\|\tilde{\zeta}_t\|^2
		\end{align*}
		The last inequality holds as we will choose $\eta\leq 1/L_G$. Summing both sides,
		\begin{align*}
		f(x_t)-f(x_0)\leq -\frac{\eta}{4}\sum_{i=0}^{t-1}\|\nabla_t\|^2 -\eta \sum_{i=0}^{t-1}\nabla_i^\top \tilde{\zeta}_i+\frac{3\eta^2L_G}{2}\sum_{i=0}^{t-1}\|\tilde{\zeta}_i\|^2 \numberthis\label{eq:totalfunctionchange}
		\end{align*}
		Observe that, by Assumption~\ref{as:SGC},
		\begin{align*}
		\mathbb{P}(\nabla_t^\top \zeta_t\geq \tau|\calF_{t-1})\leq \mathbb{P}(\|\nabla_t\| \|\zeta_t\|\geq \tau|\calF_{t-1})\leq 2\exp(-\tau^2/(\frac{2(\rho-1)}{n_1}\|\nabla_t\|^4 )) \numberthis\label{eq:innerprodconc}
		\end{align*}
		So $\nabla_t^\top \zeta_t|\calF_{t-1}$ is $c\sqrt{\frac{\rho-1}{n_1}}\|\nabla_t\|^2$-subGaussian.		
		Using Lemma~\ref{lm:innerprodsumbound}, we have, with probability at least $1-e^{-l}$,
		\begin{align*}
		-\eta\sum_{i=0}^{t-1}\nabla_i^\top\zeta_i\leq \lambda\eta c\frac{\rho-1}{n_1} \sum_{i=0}^{t-1}\|\nabla_i\|^4+\eta\frac{l}{\lambda}\leq \lambda\eta c\frac{\rho-1}{n_1}\left(\sum_{i=0}^{t-1}\|\nabla_i\|^2\right)^2+\eta\frac{l}{\lambda}
		\end{align*}
		Plugging $\lambda=\frac{32l}{\sum_{i=0}^{t-1}\|\nabla_i\|^2}$, we have,
		\begin{align*}
		-\eta\sum_{i=0}^{t-1}\nabla_i^\top\zeta_i\leq \eta \left(32cl\frac{\rho-1}{n_1}+\frac{1}{32}\right)\sum_{i=0}^{t-1}\|\nabla_i\|^2 \numberthis\label{eq:zetainnerprodsumbound}
		\end{align*}
		Using Lemma~\ref{lm:innerprodsumbound}, with probability at least $1-e^{-l}$ we have,
		\begin{align*}
		-\eta\sum_{i=0}^{t-1}\nabla_i^\top \theta_i\leq \frac{\eta}{32}\sum_{i=0}^{t-1}\|\nabla_i\|^2+32cl\eta r^2 \numberthis\label{eq:thetainnerprodsumbound}
		\end{align*}
		Using Lemma~\ref{lm:normsumbound}, we have with probability at least $1-e^{-l}$,
		\begin{align*}
		\sum_{i=0}^{t-1}\|\theta_i\|^2\leq cr^2(t+l) \numberthis\label{eq:thetanormsumbound}
		\end{align*}
		Note that by Assumption~\ref{as:SGC}, $\expec{\|\zeta_t\|^2|\calF_{t-1}}\leq \frac{\rho-1}{n_1}\|\nabla_t\|^2$, and $\|\zeta_t\|^2|\calF_{t-1}$ is $c\frac{\rho-1}{n_1}\|\nabla_t\|^2$-subExponential. So we have, with probability at least $1-e^{-l}$,
		\begin{align*}
		\sum_{i=0}^{t-1}\|\zeta_i\|^2\leq (c+l)\frac{\rho-1}{n_1}\sum_{i=0}^{t-1}\|\nabla_i\|^2 \numberthis\label{eq:zetanormsumbound}
		\end{align*}
		Combining \eqref{eq:totalfunctionchange}, \eqref{eq:zetainnerprodsumbound}, \eqref{eq:thetainnerprodsumbound}, \eqref{eq:thetanormsumbound} and \eqref{eq:zetanormsumbound}, using $\|\tilde{\zeta}_t\|^2\leq 2 (\|\zeta_t\|^2+\|\theta\|^2)$, and using union bound, we have with probability at least $1-4e^{-l}$,
		\begin{align*}
		f(x_t)-f(x_0)&\leq \left(-\frac{\eta}{4}+\eta \left(32lc\frac{\rho-1}{n_1}+\frac{1}{32}\right)+\frac{\eta}{32}+3\eta^2L_G(c+l)\frac{\rho-1}{n_1}\right)\sum_{i=0}^{t-1}\|\nabla_i\|^2\\&~+3c\eta^2r^2(t+l)L_G+32cl\eta r^2
		\end{align*}
		We need to choose $\eta$ such that $\left(-\frac{\eta}{4}+\eta \left(32lc\frac{\rho-1}{n_1}+\frac{1}{32}\right)+\frac{\eta}{32}+3\eta^2L_G(c+l)\frac{\rho-1}{n_1}\right)<-\frac{\eta}{16}$. Choosing $n_1$, and $\eta$ as in \eqref{eq:descentparchoicefirst}, and setting $t={\cT_{1}}$, we get \eqref{eq:descentfirst}.
		\item Using Lemma~\ref{lm:innerprodboundzero}, and Lemma~\ref{lm:normsumboundzero}, and Assumption~\ref{as:lip} we have, with probability at least $1-4e^{-l}$
		\begin{align*}
		&f(x_t)-f(x_0)\leq  -\frac{\eta}{4}\sum_{i=0}^{t-1}\|\nabla_i\|^2+\frac{\eta}{16}\sum_{i=0}^{t-1}\|\nabla_i\|^2+16cl\eta r^2+3cL_G\eta^2r^2(t+l)\\
		+&\frac{8\eta\sqrt{dt}}{K_1^\frac{3}{2}n_1}(t\log 9d+l)^\frac{3}{2}\sum_{i=0}^{t-1}\left(\frac{\nu LL_G(d+2)}{2}+C_0\sqrt{(\rho-1)(d+1)}\|\nabla_i\|^2\right)\\
		+&\frac{384L_Gd\eta^2t^2(t\log 9d +l)^3}{K_1^3n_1^2}\sum_{i=0}^{t-1}\left(\left(\frac{\nu L_G(d+2)}{2}\right)^2+C_0^2(\rho-1)(d+1)\|\nabla_i\|^2\right)
		\end{align*}
		We will choose ${\cT_{0}}$, $\eta$, and $n_1$ such that, \eqref{eq:parchoosecond1}, and \eqref{eq:parchoosecond2} are true. Then, with probability at least $1-4e^{-l}$, we get \eqref{eq:descentzero}.
	\end{enumerate}
\end{proof}
\begin{proof}[Proof of Lemma~\ref{lm:xtx0distnormbound}]
	\begin{enumerate}[label=\alph*)]
		\item For a fixed $\tau\leq t$, we have
		\begin{align*}
		\|x_\tau-x_0\|^2\leq \eta^2\|\sum_{i=0}^{\tau-1}(\nabla_i+\tz_i)\|^2\leq 2\eta^2 t\sum_{i=0}^{t-1}\|\nabla_i\|^2+4\eta^2(\|\sum_{i=0}^{t-1}\zeta_i\|^2+\|\sum_{i=0}^{t-1}\theta_i\|^2)
		\end{align*}
		Using Lemma~\ref{lm:sumnormbound}, we have with probability at least $1-4de^{-l}$,
		\begin{align*}
		\|\sum_{i=0}^{t-1}\zeta_i\|^2+\|\sum_{i=0}^{t-1}\theta_i\|^2\leq cl\left(\frac{\rho-1}{n_1}\sum_{i=0}^{t-1}\|\nabla_i\|^2+tr^2\right)
		\end{align*}
		Combining this with Lemma~\ref{lm:descent}, with probability at least $1-4e^{-l}-4de^{-l}$, setting $t={\cT_{1}}$, and using union bound we have \eqref{eq:xtx0distnormboundfirst}.
		\item 
		\begin{align*}
		\mathbb{P}\left(\|\sum_{i=0}^{t-1}\zeta_i\|^2\geq \tau\right)\leq \mathbb{P}\left(\sum_{i=0}^{t-1}\|\zeta_i\|^2\geq \tau/t\right)
		\end{align*}
		So from Lemma~\ref{lm:normsumboundzero}, we have with probability at least $1-e^{-l}$
		\begin{align*}
		\|\sum_{i=0}^{t-1}\zeta_i\|^2\leq \frac{128dt^3(t\log 9d +l)^3}{K_1^3n_1^2}\sum_{i=0}^{t-1}\left(\left(\frac{\nu L_G(d+2)}{2}\right)^2+C_0^2(\rho-1)(d+1)\|\nabla_i\|^2\right)
		\end{align*}
		Plugging $t={\cT_{0}}$, under condition \eqref{eq:parchoosecond1}, we have,
		\begin{align*}
		\|\sum_{i=0}^{\cT_{0}-1}\zeta_i\|^2\leq \frac{L_G\cT_{0}^2\nu^2(d+2)^2}{192C_0^2(\rho-1)(d+1)}+\frac{\cT_{0}}{12L_G}\sum_{i=0}^{\cT_{0}-1}\|\nabla_i\|^2
		\end{align*}
		From \eqref{eq:descentzero} we have, with probability at least $1-e^{-l}$
		\begin{align*}
		\sum_{i=0}^{\cT_{0}-1}\|\nabla_i\|^2\leq \frac{16}{\eta}(f(x_0)-f(x_{\cT_{0}})+\wp(r,l,\nu,\eta,d,\cT_{0}))
		\end{align*}
		Then we have with probability with at least $1-3d\cT_0e^{-l}$, we have \eqref{eq:xtx0distnormboundzero}.
	\end{enumerate}
\end{proof}
\begin{proof}[Proof of Lemma~\ref{lm:qhqsgbound}]
	\begin{enumerate}[label=\alph*)]
		\item Proof for the first-order setting is as in \cite{jin2019nonconvex}.
		\item Note that $q_h(t)$ is the same as in part (a). If we can ensure that for the zeroth-order case  $\forall t\leq {\cT_{0}}$ we have $\|q_{sg}(t+1)\|\leq \beta(t)r/(40\sqrt{d})$, then the rest of the proof follows from \cite{jin2019nonconvex}. For a fixed $t$, using Cauchy–Schwarz inequality,
		\begin{align*}
		&\mathbb{P}\left(\left\lVert q_{sg}(t+1)\right\rVert\geq\tau\right)=\mathbb{P}\left(\eta\left\lVert\sum_{i=0}^{t}(I-\eta\cH)^{t-i}\hat{\zeta}_i\right\rVert\geq\tau\right)\leq \mathbb{P}\left(\eta\sum_{i=0}^{t}\left\lVert(I-\eta\cH)\right\rVert^{t-i}\left\lVert\zeta_i-\zeta'_i\right\rVert\geq\tau\right)\\
		\leq& 2\mathbb{P}\left(\eta\sum_{i=0}^{t}\left\lVert(I-\eta\cH)\right\rVert^{t-i}\left\lVert\zeta_i\right\rVert\geq\tau/2\right)\\
		\leq& 2\mathbb{P}\left(\eta\sqrt{\sum_{i=0}^{t}\left\lVert(I-\eta\cH)\right\rVert^{2t-2i}}\sqrt{\sum_{i=0}^{t}\left\lVert\zeta_i\right\rVert^2}\geq\tau/2\right)\\
		\leq & 2\mathbb{P}\left(\sum_{i=0}^{t}\left\lVert\zeta_i\right\rVert^2\geq\left(\frac{\tau}{2\eta\beta(t+1)}\right)^2\right)
		\end{align*}
		From Lemma~\ref{lm:normsumboundzero}, we have with probability at least $1-e^{-l}$
		\begin{align*}
		&\|q_{sg}(t+1)\|\leq \frac{32\sqrt{d(t+1)}\eta \beta(t+1)((t+1)\log 9d+\log 2+l)^\frac{3}{2}}{K_1^{3/2}n_1}\\
		&\sum_{i=0}^{t}\left(\frac{\nu L_G(d+2)}{2}+C_0\sqrt{(\rho-1)(d+1)}\|\nabla_i\|\right)
		\end{align*}
Recalling the definition of $\textXi$ from~\eqref{textxidef}, and setting $t={\cT_{0}}$, we will choose ${\cT_{0}}$, $r$, $\eta$, $l$, and $\nu$ such that
		\begin{align*}
		&\textXi\cdot\sum_{i=0}^{{\cT_{0}}}\left(\frac{\nu L_G(d+2)}{2}+C_0\sqrt{(\rho-1)(d+1)}\|\nabla_i\|\right)\leq \frac{\beta({\cT_{0}})r}{40\sqrt{d}} \numberthis\label{eq:parchoosecond3}
		\end{align*}
	\end{enumerate}
\end{proof}
\begin{proof}[Proof of Lemma~\ref{lm:escapesaddle}]
	\begin{enumerate}[label=\alph*)]
		\item For the first part, we have from Lemma~\ref{lm:descent}, with probability at least $1-4e^{-l}$,
		\begin{align*}
		f(x_{\cT_{1}})-f(x_0)\leq 3c\eta^2r^2({\cT_{1}}+l)L_G+32cl\eta r^2\leq 0.1\cF_1
		\end{align*}
		By similar methods in \cite{jin2019nonconvex}, we have, with probability at least $2/3-10d{\cT_{1}}^2\log\left(\frac{\cS_1\sqrt{d}}{\eta r}\right)e^{-l}$, if $\min\{f(x_{\cT_{1}})-f(x_0),f(x_{\cT_{1}}')-f(x_0)\}\geq -\cF_1$, then 
		\begin{align*}
		\max\{\|x_{\cT_{1}}-x_0\|,\|x_{\cT_{1}}'-x_0\|\}\geq \frac{\beta({\cT_{1}})\eta r }{40\sqrt{d}}=\frac{(1+\eta \gamma)^{\cT_{1}}\sqrt{\eta} r}{40\sqrt{2\gamma d}}>\cS_1 \numberthis\label{eq:parchoosecond5}
		\end{align*}
		This is in contradiction with Lemma~\ref{lm:improvorloc}. Then we have with probability at least $2/3-10d{\cT_{1}}^2\log\left(\frac{\cS_1\sqrt{d}}{\eta r}\right)e^{-l}$, $\min\{f(x_{\cT_{1}})-f(x_0),f(x_{\cT_{1}}')-f(x_0)\}\leq -\cF_1$. As the marginal distributions of $x_{\cT_1}$ and $x'_{\cT_1}$ are same we have,
		\begin{align*}
		&\mathbb{P}(f(x_{\cT_{1}}')-f(x_0)\}\leq -\cF_1)\geq \frac{1}{2}\mathbb{P}(\min\{f(x_{\cT_{1}})-f(x_0),f(x_{\cT_{1}}')-f(x_0)\}\leq -\cF_1)\\
		\geq& 1/3-9d{\cT_{1}}^2\log\left(\frac{\cS_1\sqrt{d}}{\eta r}\right)e^{-l}
		\end{align*}
		\item Note that the probability for the second statement being true is at least $1/3-1.5{\cT_{0}}^2e^{-l}$ which is different from \cite{jin2019nonconvex} but the proof method is same. So we omit the proof here. 
	\end{enumerate}
\end{proof}

\section{Proof of Theorem~\ref{th:crnmaintheorem}}\label{sec:crnproof}

We first state the following optimality conditions for CR Newton method updates due to \cite{nesterov2006cubic}.
\begin{lemma}\cite{nesterov2006cubic}
	\begin{subequations}
		\begin{align}
		\begin{split}\label{eq:optconda}
		g_t+H_th_t^*+\frac{M}{2}\|h_t^*\|h_t=0
		\end{split}\\
		\begin{split} \label{eq:optcondb}
		H_t+\frac{M}{2}\|h_t^*\|I\succcurlyeq 0
		\end{split}
		\end{align}
	\end{subequations}
\end{lemma}
Intuitively, the proof follows through three stages. First, in Lemma~\ref{lm:envelopdescent}, we show that the descent at each time point is proportional to the cube of the step size. 
\begin{lemma}\label{lm:envelopdescent}\cite{tripuraneni2018stochastic}
	Let $m_t$ be as defined in \eqref{eq:envelope}. Then for all $t$,
	\begin{align}
	m_t(x_t+h_t^*)-m_t(x_t)\leq -\frac{M}{12} \|h_t^*\|^3 \label{eq:envelopdescent}
	\end{align}
\end{lemma}

Then, in Lemma~\ref{lm:htlowerbound} we show that the second-order staionarity of an iterate is upper bounded by the step size at that time point. 
\begin{lemma} \label{lm:htlowerbound}
	Let Assumption~\ref{as:lipgrad}, and \ref{as:liphess} hold true for $f$. Then the following holds  $\forall t$
	\begin{enumerate}[label=\alph*)]
		\item for the first-order update of a CR Newton method,
		\begin{align*}
		&\sqrt{\expec{\|h_t^*\|^2|\calF_t}}
		\geq \max\left(\left(A\expec{\|\nabla f\left(x_t+h_t^*\right)\||\calF_t}\right.
		\left. -B\right)^\frac{1}{2},\right.\\
		&\left.\frac{2}{M+2L_H}\left(-\sqrt{\frac{\sigma_2^2}{n_2}}-\expec{\lambda_{1,t+1}|\calF_t}\right)\right) \numberthis\label{eq:htlowerbound}
		\end{align*}
		where $A=\frac{1}{2(L_H+M)}\left(1-\sqrt{\frac{\rho-1}{n_1}}\right)$, and $B=\frac{1}{4(L_H+M)^2}\left(\frac{\rho-1}{2n_1}L_G^2+\frac{\sigma_2^2}{n_2}\right)$.
		\item for the zeroth-order update of a CR Newton Method
		\begin{align*}
		&\sqrt{\expec{\|h_t^*\|^2|\calF_t}}
		\geq \max\left(\left(A'\expec{\|\nabla f\left(x_t+h_t^*\right)\||\calF_t}\right.
		\left. -B'\right)^\frac{1}{2},\right.\\
		&\left.\frac{2}{M+2L_H}\left(-\sqrt{\frac{128(1+2\log 2d)(d+16)^4L_G^2}{3n_2}}-\sqrt{3}\nu L_H(d+16)^\frac{5}{2}-\expec{\lambda_{1,t+1}|\calF_t}\right)\right) \numberthis\label{eq:htlowerboundzero}
		\end{align*}
		where $A'=\frac{1}{2(L_H+M)}\left(1-\sqrt{\frac{\rho'-1}{n_1}}\right)$, and\\ $B'=\frac{1}{4(L_H+M)^2}\left(\frac{\rho'-1}{n_1}L_G^2
		+\frac{128(1+2\log 2d)(d+16)^4L_G^2}{3n_2}+3L_H^2\nu^2(d+16)^5+\sqrt{6}\nu (L_H+M)L_G(d+3)^\frac{3}{2}\right)$.
	\end{enumerate}
\end{lemma}
Finally, in Lemma~\ref{lm:htcubeupperbound}, we prove that the expected step size becomes smaller with the horizon. 
\begin{lemma}\label{lm:htcubeupperbound}
	Let $f$ be a function for which Assumptions~\ref{as:lipgrad}, and \ref{as:liphess} are true. Then,
	\begin{enumerate}[label=\alph*)]
		\item for first-order updates generated by Algorithm~\ref{alg:nccrn} the following holds:
		\begin{align*}
		&\left(\frac{M}{72}-\left(\frac{\rho-1}{n_1}\right)^\frac{3}{4}\frac{8}{\sqrt{M}A^\frac{3}{2}}\right)\expec{\|h_R^*\|^3|\calF_t}\\
		&\leq \frac{f(x_1)-f^*}{T} +\frac{1152L_G^3}{M^2}\left(\frac{\rho-1}{n_1}\right)^\frac{3}{2}\\
		&+\frac{8}{\sqrt{M}}\left(\frac{\rho-1}{n_1}\right)^\frac{3}{4}\left(\frac{B}{A}\right)^\frac{3}{2} +\frac{324}{M^2}\frac{\sigma_2^3}{n_2^{3/2}} \numberthis\label{eq:htcubeupperbound}
		\end{align*}
		where $R$ is an integer random variable uniformly distributed over the support $\lbrace 1,2,\cdots,T\rbrace$.
		\item for zeroth-order updates generated by Algorithm~\ref{alg:nccrn} the following holds:
		\begin{align*}
		&\left(\frac{M}{144}-\left(\frac{\rho'-1}{n_1}\right)^\frac{3}{4}\frac{6}{\sqrt{M}{A'}^\frac{3}{2}}\right)\expec{\|h_R^*\|^3|\calF_t}\\
		&\leq \frac{f(x_1)-f^*}{T} +\frac{864L_G^3}{M^2}\left(\frac{\rho'-1}{n_1}\right)^\frac{3}{2}+\frac{4}{M}(\nu L_G)^\frac{3}{2}(d+3)^\frac{9}{4}\\
		&+\frac{6}{\sqrt{M}}\left(\frac{\rho'-1}{n_1}\right)^\frac{3}{4}\left(\frac{B'}{A'}\right)^\frac{3}{2} +\frac{162}{M^2}\left(\frac{160\sqrt{1+2\log 2d}(d+16)^6L_G^3}{n_2^\frac{3}{2}}+21L_H^3(d+16)^\frac{15}{2}\nu^3 \right) \numberthis\label{eq:htcubeupperboundzero}
		\end{align*}
		where $R$ is an integer random variable uniformly distributed over the support $\lbrace 1,2,\cdots,T\rbrace$.
	\end{enumerate}
\end{lemma}
Combining the above three facts, we complete proof of  Theorem~\ref{th:crnmaintheorem}.
\begin{proof}[Proof of Theorem~\ref{th:crnmaintheorem}]
	\begin{enumerate}[label=\alph*)]
		\item From Lemma~\ref{lm:htlowerbound} we have,
		\begin{align*}
		&\sqrt{\expec{\|h_t^*\|^2|\calF_t}}+\sqrt{B}+\frac{2}{(2L_H+M)}\sqrt{\frac{\sigma_2^2}{n_2}}
		\geq \\
		&\max\left(\sqrt{A\expec{\|\nabla f\left(x_t+h_t^*\right)\||\calF_t}},-\frac{2}{(2L_H+M)}\expec{\lambda_{1,t+1}|\calF_t}\right) \numberthis\label{eq:htlowerboundrearranged}
		\end{align*}
		From Lemma~\ref{lm:htcubeupperbound}, we have
		\begin{align*}
		&\left(\left(\frac{M}{72}-\left(\frac{\rho-1}{n_1}\right)^\frac{3}{4}\frac{8}{\sqrt{M}A^\frac{3}{2}}\right)\expec{\|h_R^*\|^3|\calF_t}\right)^\frac{1}{3}\\
		&\leq \left(\frac{f(x_1)-f^*}{T}\right)^\frac{1}{3} +\frac{11L_G}{M^\frac{2}{3}}\left(\frac{\rho-1}{n_1}\right)^\frac{1}{2}\\
		&+\frac{2}{M^\frac{1}{6}}\left(\frac{\rho-1}{n_1}\right)^\frac{1}{4}\left(\frac{B}{A}\right)^\frac{1}{2} +\frac{7}{M^\frac{2}{3}}\frac{\sigma_2}{n_2^{1/2}}
		\numberthis\label{eq:htcubeonethirdupperbound}
		\end{align*}
		Combining \eqref{eq:htlowerboundrearranged} with \eqref{eq:htcubeonethirdupperbound}, using Jensens's inequality we have, and choosing $n_1$, $n_2$, $T$, and $M$ as in \eqref{eq:crnewtonparameterchoice}, we have  $\max\left(\sqrt{\frac{\expec{\|\nabla f\left(x_R\right)\|}}{144M}},-\frac{\expec{\lambda_{1,R}}}{9M}\right)\leq \sqrt{\epsilon} $. Total number of first-order oracle calls, and second-order oracle calls are $Tn_1=Tn_2=\order\left(\frac{1}{\epsilon^\frac{5}{2}}\right)$.
		\item From Lemma~\ref{lm:htlowerbound} we have,
		\begin{align*}
		&\sqrt{\expec{\|h_t^*\|^2|\calF_t}}+\sqrt{B'}+\frac{2}{(2L_H+M)}\left(\sqrt{\frac{128(1+2\log 2d)(d+16)^4L_G^2}{3n_2}}+\sqrt{3}\nu L_H(d+16)^\frac{5}{2}\right)
		\geq \\
		&\max\left(\sqrt{A'\expec{\|\nabla f\left(x_t+h_t^*\right)\||\calF_t}},-\frac{2}{(2L_H+M)}\expec{\lambda_{1,t+1}|\calF_t}\right) \numberthis\label{eq:htlowerboundrearrangedzero}
		\end{align*}
		From Lemma~\ref{lm:htcubeupperbound}, we have
		\begin{align*}
		&\left(\left(\frac{1}{144}-\left(\frac{\rho'-1}{n_1}\right)^\frac{3}{4}\frac{6}{{M}^\frac{3}{2}{A'}^\frac{3}{2}}\right)\expec{\|h_R^*\|^3|\calF_t}\right)^\frac{1}{3}\\
		&\leq \left(\frac{f(x_1)-f^*}{MT}\right)^\frac{1}{3} +\frac{10L_G}{M}\left(\frac{\rho'-1}{n_1}\right)^\frac{1}{2}+\frac{2}{M^\frac{2}{3}}(\nu L_G)^\frac{1}{2}(d+3)^\frac{3}{4}\\
		&+\frac{2}{{M}^\frac{1}{2}}\left(\frac{\rho'-1}{n_1}\right)^\frac{1}{4}\left(\frac{B'}{A'}\right)^\frac{1}{2} +\frac{6}{M}\left(\frac{6L_G(1+2\log 2d)^\frac{1}{6}(d+16)^2}{n_2^\frac{1}{2}}+3L_H(d+16)^\frac{5}{2}\nu \right)  \numberthis\label{eq:htcubeonethirdupperboundzero}
		\end{align*}
		Combining \eqref{eq:htlowerboundrearrangedzero} with \eqref{eq:htcubeonethirdupperboundzero}, using Jensens's inequality we have, and choosing $n_1$, $n_2$, $T$, $\nu$, and $M$ as in \eqref{eq:crnewtonparameterchoicezero}, we have  $\max\left(\sqrt{\expec{\|\nabla f\left(x_R\right)\|}},-\expec{\lambda_{1,R}}\right)\leq \order\left(\sqrt{\epsilon}\right)$. Total number of first-order oracle calls is $Tn_1=\order\left(\frac{d}{\epsilon^\frac{5}{2}}\right)$, and second-order oracle calls is $Tn_2=\order\left(\frac{d^4\log d}{\epsilon^\frac{5}{2}}\right)$.
	\end{enumerate}
\end{proof}

\subsection{Proofs of Lemmas related to CR Newton method}
\begin{lemma}\label{lm:hessestvar}\cite{roy2019multipoint}
\begin{align}
    &\expec{\left\lVert\nabla_t^2-\frac{1}{n_2}\sum_{i=1}^{n_2}\nabla^2 F(x_t,\xi_i)\right\rVert^2}\leq \frac{\sigma_2^2}{n_2}\\
    &\expec{\left\lVert\nabla_t^2-\frac{1}{n_2}\sum_{i=1}^{n_2}\nabla^2 F(x_t,\xi_i)\right\rVert^3}\leq \frac{2\sigma_2^3}{n_2^\frac{3}{2}}
\end{align}
\end{lemma}
For the zeroth-order estimates of gradient and Hessian as defined in \eqref{eq:zerohessdef} we have the following concentration result. 
\begin{lemma}\cite{balasubramanian2018zeroth}\label{lm:zerothorderesterror}
\begin{subequations}
\begin{align}
\begin{split}
    \expec{\|g_t-\nabla_t\|^2}\leq \frac{\rho'-1}{n_1}\|\nabla_t\|^2+\frac{3\nu^2}{2}L_G^2(d+3)^3 \label{eq:gradesterror}
\end{split}\\
\begin{split}
    \expec{\|\nabla_t^2-H_t\|^2}\leq \frac{128(1+2\log 2d)(d+16)^4L_G^2}{3n_2}+3L_H^2(d+16)^5\nu^2 \label{eq:hessesterror2}
\end{split}\\
\begin{split}
    \expec{\|\nabla_t^2-H_t\|^3}\leq \frac{160\sqrt{1+2\log 2d}(d+16)^6L_G^3}{n_2^\frac{3}{2}}+21L_H^3(d+16)^\frac{15}{2}\nu^3 \label{eq:hessesterror3}
\end{split}
\end{align}
\end{subequations}
where $\rho'=1+4(d+5)\rho$
\end{lemma}
\begin{proof}[Proof of Lemma~\ref{lm:htlowerbound}]
	\begin{enumerate}[label=\alph*)]
		\item Using \eqref{eq:optconda} we get,
		\begin{align*}
		\|g_t+H_t h_t^*\|=\frac{M}{2}\|h_t^*\|^2
		\end{align*}
		Then, using Assumption~\ref{as:liphess}, and Young's inequality we get,
		\begin{align*}
		&\|\nabla f\left(x_t+h_t^*\right)\|\\
		\leq &\|\nabla f\left(x_t+h_t^*\right)-\nabla_t -\nabla_t^2 h_t^*\|+\|\nabla_t+\nabla^2_th_t^*\|\\
		\leq & \|\nabla f\left(x_t+h_t^*\right)-\nabla_t-\nabla^2_th_t^*\|+\|g_t+H_th_t^*\|\\
		+&\|g_t-\nabla_t\|+\|\left(H_t-\nabla^2_t\right)h_t^*\|\\
		\leq & \frac{M+L_H}{2}\|h_t^*\|^2+\|g_t-\nabla_t\|+\|\left(H_t-\nabla^2_t\right)h_t^*\|\\
		\leq & \left(M+L_H\right)\|h_t^*\|^2+\|g_t-\nabla_t\|+\frac{1}{2(L_H+M)}\|H_t-\nabla^2_t\|^2
		\end{align*}
		Taking expectation on both sides, and using Lemma~\ref{lm:gradestvar}, Lemma~\ref{lm:hessestvar}, and Jensen's inequality we have
		\begin{align*}
		&\expec{\|\nabla f\left(x_t+h_t^*\right)\||\calF_t}\\
		&\leq (L_H+M)\expec{\|h_t^*\|^2|\calF_t}+\sqrt{\frac{\rho-1}{n_1}}\|\nabla_t\|
		+\frac{\sigma_2^2}{2(L_H+M)n_2}
		\end{align*}
		Using Assumption~\ref{as:lipgrad} we get
		\begin{align*}
		&\left(1-\sqrt{\frac{\rho-1}{n_1}}\right)\expec{\|\nabla f\left(x_t+h_t^*\right)\||\calF_t}\\
		&\leq (L_H+M)\expec{\|h_t^*\|^2|\calF_t}+\sqrt{\frac{\rho-1}{n_1}}L_G
		\expec{\|h_t^*\||\calF_t}+\frac{\sigma_2^2}{2(L_H+M)n_2}\\
		&\leq 2(L_H+M)\expec{\|h_t^*\|^2|\calF_t}+\frac{1}{2(L_H+M)}\left(\frac{\rho-1}{n_1}L_G^2
		+\frac{\sigma_2^2}{n_2}\right)
		\end{align*}
		Rearranging we have,
		\begin{align*}
		&\sqrt{\expec{\|h_t^*\|^2|\calF_t}}\\
		&\geq \left(\frac{1}{2(L_H+M)}\left(1-\sqrt{\frac{\rho-1}{n_1}}\right)\expec{\|\nabla f\left(x_t+h_t^*\right)\||\calF_t}\right.\\
		&\left. -\frac{1}{4(L_H+M)^2}\left(\frac{\rho-1}{n_1}L_G^2
		+\frac{\sigma_2^2}{n_2}\right)\right)^\frac{1}{2}\numberthis \label{eq:phtlowerboundproofa}
		\end{align*}
		Now, using Assumption~\ref{as:liphess} we get
		\begin{align*}
		&\expec{\nabla^2f\left(x_t+h_t^*\right)|\calF_t}\succcurlyeq \expec{\nabla_t^2-L_H\|h_t^*\|I|\calF_t}\\
		&\succcurlyeq \expec{H_t-L_H\|h_t^*\|I|\calF_t} -\sqrt{\frac{\sigma_2^2}{n_2}}I\\
		&\succcurlyeq -\sqrt{\frac{\sigma_2^2}{n_2}}I-\left(L_H+\frac{M}{2}\right)\expec{\|h_t^*\||\calF_t}I\\
		&\expec{\|h_t^*\||\calF_t}\geq \frac{2}{M+2L_H}\left(-\sqrt{\frac{\sigma_2^2}{n_2}}-\expec{\lambda_{1,t+1}|\calF_t}\right)\numberthis \label{eq:phtlowerboundproofb}
		\end{align*}
		Now using Jensen's inequality, and \eqref{eq:phtlowerboundproofa} we get \eqref{eq:htlowerbound}.
		\item
		Using Lemma~\ref{lm:zerothorderesterror}, and following the proof of part (a), \eqref{eq:phtlowerboundproofa} becomes
		\begin{align*}
		&\sqrt{\expec{\|h_t^*\|^2|\calF_t}}
		\geq \left(\frac{1}{2(L_H+M)}\left(1-\sqrt{\frac{\rho'-1}{n_1}}\right)\expec{\|\nabla f\left(x_t+h_t^*\right)\||\calF_t}\right.\\
		&\left. -\frac{1}{4(L_H+M)^2}\left(\frac{\rho'-1}{n_1}L_G^2
		+\frac{128(1+2\log 2d)(d+16)^4L_G^2}{3n_2}+3L_H^2\nu^2(d+16)^5+\sqrt{6}\nu (L_H+M)L_G(d+3)^\frac{3}{2}\right)\right)^\frac{1}{2}\numberthis \label{eq:phtlowerboundproofazero}
		\end{align*}
		Similarly, \eqref{eq:phtlowerboundproofb} becomes
		\begin{align*}
		\expec{\|h_t^*\||\calF_t}\geq \frac{2}{(2L_H+M)}\left(-\sqrt{\frac{128(1+2\log 2d)(d+16)^4L_G^2}{3n_2}}-\sqrt{3}\nu L_H(d+16)^\frac{5}{2}-\expec{\lambda_{1,t+1}|\calF_t}\right)\numberthis \label{eq:phtlowerboundproofbzero}
		\end{align*}
	\end{enumerate}
\end{proof}
\begin{proof}[Proof of Lemma~\ref{lm:htcubeupperbound}]
	\begin{enumerate}[label=\alph*)]
		\item Using Young's inequality, and \eqref{eq:envelopdescent}, we get 
		\begin{align*}
		&f(x_t+h_t^*)-f(x_t)\leq m_t( x_t+h_t^*)-m_t(x_t)\\
		&+(\nabla_t-g_t)^\top h_t^*
		+\frac{1}{2}{h_t^*}^\top(\nabla_t^2-H_t)h_t^* \\
		&\leq m_t( x_t+h_t^*)-m_t(x_t)\\
		&+\frac{4}{\sqrt{3M}}\|\nabla_t-g_t\|^\frac{3}{2}
		+\frac{162}{M^2}\|\nabla_t^2-H_t\|^3+\frac{M}{18}\| h_t^*\|^3\\
		&\leq - \frac{M}{36}\| h_t^*\|^3+\frac{4}{\sqrt{3M}}\|\nabla_t-g_t\|^\frac{3}{2}
		+\frac{162}{M^2}\|\nabla_t^2-H_t\|^3
		\end{align*}
		Taking expectation on both sides, and using Lemma~\ref{lm:gradestvar} with Jensen's inequality, and Lemma~\ref{lm:hessestvar}, we get
		\begin{align*}
		&\expec{f(x_t+h_t^*)|\calF_t}-f(x_t)\leq -\frac{M}{36}\expec{\|h_t^*\|^3|\calF_t}\\
		&+\frac{4}{\sqrt{3M}}\left(\frac{\rho-1}{n_1}\right)^\frac{3}{4}\|\nabla_t\|^\frac{3}{2}+\frac{162}{M^2}\frac{2\sigma_2^3}{n_2^{3/2}} \numberthis\label{eq:progressintermed}
		\end{align*}
		Now let us relate the gradient size $\|\nabla_t\|$ with $\|h_t^*\|$. Note that, as $x_{t+1}=x_t+h_t^*$ we will use $\nabla_{t+1}$ to denote $\nabla f(x_t+h_t^*)$ here. Using triangle inequality, the fact $(a+b)^{3/2}\leq \sqrt{2}(a^{3/2}+b^{3/2})$ for $a,b>0$, Assumption~\ref{as:lipgrad}, and Jensen's inequality we get 
		\begin{align*}
		&\|\nabla_t\|^\frac{3}{2}=\|\nabla_t-\expec{\nabla_{t+1}|\calF_t}+\expec{\nabla_{t+1}|\calF_t}\|^\frac{3}{2}\\
		\leq & (\|\nabla_t-\expec{\nabla_{t+1}|\calF_t}\|+\|\expec{\nabla_{t+1}|\calF_t}\|)^\frac{3}{2}\\
		\leq & \sqrt{2}(\|\nabla_t-\expec{\nabla_{t+1}|\calF_t}\|^\frac{3}{2}+\|\expec{\nabla_{t+1}|\calF_t}\|^\frac{3}{2})\\
		\leq & \sqrt{2}(L_G^\frac{3}{2}\expec{\|h_t^*\|^\frac{3}{2}|\calF_t}+\expec{\|\nabla_{t+1}\||\calF_t}^\frac{3}{2})\numberthis\label{eq:delt32boundintermed}
		\end{align*}
		From Lemma~\ref{lm:htlowerbound} we have,
		\begin{align*}
		\expec{\|h_t^*\|^2|\calF_t}+B
		\geq A\expec{\|\nabla_{t+1} \||\calF_t}
		\end{align*}
		Again using the fact $(a+b)^{3/2}\leq \sqrt{2}(a^{3/2}+b^{3/2})$ for $a,b>0$, and Jensens's inequality we get
		\begin{align}\label{eq:cubehtlowerbound}
		\sqrt{2}\left(\expec{\|h_t^*\|^3|\calF_t}+B^\frac{3}{2}\right)
		\geq \left(A\expec{\|\nabla_{t+1} \||\calF_t}\right)^\frac{3}{2}
		\end{align}
		Combining \eqref{eq:delt32boundintermed}, and \eqref{eq:cubehtlowerbound}, we get
		\begin{align*}
		\|\nabla_t\|^\frac{3}{2}\leq & \sqrt{2}L_G^\frac{3}{2}\expec{\|h_t^*\|^\frac{3}{2}|\calF_t}+\frac{2}{A^\frac{3}{2}}\expec{\|h_t^*\|^3|\calF_t}\\
		+&2\left(\frac{B}{A}\right)^\frac{3}{2}
		\end{align*}
		Now, using Young's inequality
		\begin{align*}
		&\left(\frac{\rho-1}{n_1}\right)^\frac{3}{4}\|\nabla_t\|^\frac{3}{2}\leq \frac{288L_G^3}{M^\frac{3}{2}}\left(\frac{\rho-1}{n_1}\right)^\frac{3}{2}\\
		&+\frac{\sqrt{3}M^\frac{3}{2}}{288}\expec{\|h_t^*\|^3|\calF_t}+\left(\frac{\rho-1}{n_1}\right)^\frac{3}{4}\frac{2}{A^\frac{3}{2}}\expec{\|h_t^*\|^3|\calF_t}\\
		&+2\left(\frac{\rho-1}{n_1}\right)^\frac{3}{4}\left(\frac{B}{A}\right)^\frac{3}{2} \numberthis\label{eq:rhonablatbound}
		\end{align*}
		Combining \eqref{eq:progressintermed}, and \eqref{eq:rhonablatbound} we get
		\begin{align*}
		&\expec{f(x_t+h_t^*)|\calF_t}-f(x_t)\leq -\frac{M}{72}\expec{\|h_t^*\|^3|\calF_t}\\
		&+\frac{1152L_G^3}{M^2}\left(\frac{\rho-1}{n_1}\right)^\frac{3}{2}\\
		&+\left(\frac{\rho-1}{n_1}\right)^\frac{3}{4}\frac{8}{\sqrt{M}A^\frac{3}{2}}\expec{\|h_t^*\|^3|\calF_t}\\
		&+\frac{8}{\sqrt{M}}\left(\frac{\rho-1}{n_1}\right)^\frac{3}{4}\left(\frac{B}{A}\right)^\frac{3}{2} +\frac{324}{M^2}\frac{\sigma_2^3}{n_2^{3/2}} \numberthis\label{eq:crndescentbeforesum}
		\end{align*}
		Rearranging and summing from $t=1$ to $T$, and dividing both sides by $T$ we get \eqref{eq:htcubeupperbound}.
		\item Using Lemma~\ref{lm:zerothorderesterror}, and following the proof of Lemma~\ref{lm:htcubeupperbound} we have the following inequality corresponding to \eqref{eq:progressintermed}
		\begin{align*}
		&\expec{f(x_t+h_t^*)|\calF_t}-f(x_t)\leq -\frac{M}{36}\expec{\|h_t^*\|^3|\calF_t}\\
		&+\frac{3}{\sqrt{M}}\left(\frac{\rho'-1}{n_1}\right)^\frac{3}{4}\|\nabla_t\|^\frac{3}{2}+\frac{4}{M}(\nu L_G)^\frac{3}{2}(d+3)^\frac{9}{4}\\
		&+\frac{162}{M^2}\left(\frac{160\sqrt{1+2\log 2d}(d+16)^6L_G^3}{n_2^\frac{3}{2}}+21L_H^3(d+16)^\frac{15}{2}\nu^3 \right) \numberthis\label{eq:progressintermedzero}
		\end{align*}
		Eventually we get the following descent in the function value similar to \eqref{eq:crndescentbeforesum}
		\begin{align*}
		&\expec{f(x_t+h_t^*)|\calF_t}-f(x_t)\leq -\frac{M}{144}\expec{\|h_t^*\|^3|\calF_t}\\
		&+\frac{864L_G^3}{M^2}\left(\frac{\rho'-1}{n_1}\right)^\frac{3}{2}+\left(\frac{\rho'-1}{n_1}\right)^\frac{3}{4}\frac{6}{\sqrt{M}{A'}^\frac{3}{2}}\expec{\|h_t^*\|^3|\calF_t}+\frac{4}{M}(\nu L_G)^\frac{3}{2}(d+3)^\frac{9}{4}\\
		&+\frac{6}{\sqrt{M}}\left(\frac{\rho'-1}{n_1}\right)^\frac{3}{4}\left(\frac{B'}{A'}\right)^\frac{3}{2} +\frac{162}{M^2}\left(\frac{160\sqrt{1+2\log 2d}(d+16)^6L_G^3}{n_2^\frac{3}{2}}+21L_H^3(d+16)^\frac{15}{2}\nu^3 \right) \numberthis\label{eq:crndescentbeforesumzero}
		\end{align*}
		Rearranging and summing from $t=1$ to $T$, and dividing both sides by $T$ we get \eqref{eq:htcubeupperboundzero}.
	\end{enumerate}
\end{proof}

\bibliography{interpolbib}
\bibliographystyle{amsalpha}
\end{document}